\documentclass{article}
\usepackage{amsfonts,amsmath,amsthm,amssymb}

\date{}

\usepackage{algorithm}
\usepackage{algorithmic}
\usepackage{subfigure}
\usepackage{epsfig}
\usepackage{graphicx}
\usepackage{url}
\usepackage{multirow}
\usepackage{amssymb}
\usepackage{mathrsfs}
\usepackage{epstopdf}
\usepackage{multicol}
\usepackage{amsmath}
\usepackage{mathtools}
\usepackage{bm}

\usepackage{multicol}
\usepackage{wrapfig}
\usepackage{wrapfig,lipsum,booktabs}

\usepackage{multirow}

\usepackage{pifont}

\allowdisplaybreaks[4]

\newtheorem{theorem}{\textbf{Theorem}}
\newtheorem{assumption}{\textbf{Assumption}}
\newtheorem{lemma}{\textbf{Lemma}}
\newtheorem{corollary}{\textbf{Corollary}}
\newtheorem{remark}{\textbf{Remark}}

\numberwithin{equation}{section}

\theoremstyle{plain}

\theoremstyle{definition}

\title{Decentralized Stochastic Gradient Descent Ascent  for \\Finite-Sum Minimax Problems}

\author{
	Hongchang Gao \thanks{Temple University, {\tt hongchang.gao@temple.edu}} 
	}
\begin{document}
\maketitle

%
\begin{abstract}
	Minimax optimization problems have attracted significant attention in recent years due to their widespread application in numerous machine learning models. To solve the minimax problem, a wide variety of stochastic optimization methods have been proposed. However, most of them ignore the distributed setting where the training data is distributed on multiple workers. In this paper, we developed a novel decentralized stochastic gradient descent ascent method for the finite-sum minimax problem. In particular, by employing the variance-reduced gradient, our method can achieve $O(\frac{\sqrt{n}\kappa^3}{(1-\lambda)^2\epsilon^2})$ sample complexity and $O(\frac{\kappa^3}{(1-\lambda)^2\epsilon^2})$ communication complexity for the nonconvex-strongly-concave minimax problem. As far as we know, our work is the first one to achieve such theoretical complexities for this kind of minimax problem.  At last, we apply our method to   AUC maximization, and the experimental results confirm the effectiveness of our method. 
\end{abstract}

\section{Introduction}
In this paper, we consider the following decentralized finite-sum minimax problem:
\begin{equation} \label{loss}
	\begin{aligned}
		& \min_{\mathbf{x}\in \mathbb{R}^d} \max_{\mathbf{y}\in \mathbb{R}^{d'}} F(\mathbf{x}, \mathbf{y}) \triangleq \frac{1}{K}\sum_{k=1}^{K}\Big(\frac{1}{n}\sum_{i=1}^{n} f_i^{(k)}(\mathbf{x}, \mathbf{y})\Big) \ .
	\end{aligned}
\end{equation}
It is assumed that there are totally $K$ workers in a decentralized training system. Each worker has its own dataset and  objective function $f^{(k)}(\mathbf{x}, \mathbf{y})=\frac{1}{n}\sum_{i=1}^{n} f_i^{(k)}(\mathbf{x}, \mathbf{y})$ where $f_i^{(k)}(\mathbf{x}, \mathbf{y})$ is the loss function for the $i$-th sample on the $k$-th worker and $n$ is the total number of samples on each worker.  In this paper,  $f_i^{(k)}(\mathbf{x}, \mathbf{y})$ is assumed to be nonconvex in $\mathbf{x}$ and $\mu$-strongly-concave in $\mathbf{y}$.  
Under this kind of decentralized setting, all workers  collaboratively optimize Eq.~(\ref{loss}) to learn the model parameter $\mathbf{x}$ and $\mathbf{y}$.

The minimax optimization problem in Eq.~(\ref{loss}) covers numerous machine learning models, such as  adversarial training \cite{goodfellow2014generative,goodfellow2014explaining,madry2017towards}, distributionally robust optimization \cite{lin2020gradient,luo2020stochastic}, AUC maximization \cite{ying2016stochastic,liu2019stochastic}, etc. Recently, many efforts have been devoted to developing efficient optimization algorithms to solve the minimax optimization problem. For instance, \cite{lin2020gradient} proposed a stochastic gradient descent ascent method and investigated its convergence rate. Afterwards, several accelerated methods \cite{luo2020stochastic,xu2020enhanced,qiu2020single} have been proposed to improve the convergence speed by utilizing the variance reduction technique or momentum strategy.  However,  these methods only focus on the single-machine setting. It's unclear how these methods converge under the decentralized setting and how large their communication complexities are.

To handle the large-scale minimax optimization problem,  some distributed methods have been proposed in recent years. In \cite{deng2021local}, a communication-efficient local stochastic gradient descent ascent method was proposed, whose convergence rate was further improved in \cite{xie2021federated} by resorting to the variance reduction technique. \cite{guo2020communication} deveoped CoDA for the AUC maximization problem. However, these methods are based on the parameter-server setting so that they are not applicable to our decentralized setting. Recently,  \cite{liu2019decentralized} developed a decentralized optimistic stochastic gradient method and established the convergence rate for  the nonconvex-nonconcave problem. \cite{xian2021faster} developed a decentralized stochastic variance-reduced gradient descent ascent method for the nonconvex-strongly-concave problem based on the STORM gradient estimator \cite{cutkosky2019momentum}.  However, it has a large communication complexity $O(1/\epsilon^3)$  to achieve the $\epsilon$-accuracy solution \footnote{Here, we omit the spectral gap and condition number for simplification.}.   On the contrary, the decentralized algorithm for minimization problems can achieve the  $O(1/\epsilon^2)$ communication complexity. Moreover, \cite{xian2021faster}  only studied the stochastic setting, failing to handle the finite-sum optimization problem.  Recently, \cite{zhang2021taming} reformulated the policy evaluation problem in reinforcement learning as a finite-sum minimax problem and then proposed the decentralized GT-SRVR method to solve it. However, this method requires to compute the full gradient periodically,  incurring  large computation overhead.

To overcome  aforementioned issues,  we developed a novel decentralized stochastic gradient descent ascent (DSGDA) method for optimizing Eq.~(\ref{loss}) efficiently.  In detail, on each worker, DSGDA computes the variance-reduced gradient based on the local dataset and then employs the gradient tracking communication scheme to update the local model parameters $\mathbf{x}$ and $\mathbf{y}$.  Furthermore, we established the convergence rate of DSGDA for the finite-sum nonconvex-strongly-concave problem. Specifically, our theoretical analysis shows that DSGDA can achieve  $O(\frac{\kappa^3}{(1-\lambda)^2\epsilon^2})$ communication complexity, which is better than  $O(\frac{\kappa^3}{(1-\lambda)^2\epsilon^3})$ of \cite{xian2021faster} and matches $O(\frac{\kappa^3}{(1-\lambda)^2\epsilon^2})$ of  \cite{zhang2021taming} in terms of the order of the solution accuracy $\epsilon$, where $1-\lambda$ represents the spectral gap of the communication network and $\kappa$ denotes the condition number of the loss function. 
Moreover,  our method can achieve   $O(\frac{\sqrt{n}\kappa^3}{(1-\lambda)^2\epsilon^2})$ sample complexity on each worker, which is better than $O(n+\frac{\sqrt{n}\kappa^3}{(1-\lambda)^2\epsilon^2})$ of \cite{zhang2021taming}\footnote{The first term is ignored in \cite{zhang2021taming}. } in terms of $n$ because our method does not need to periodically compute the full gradient as \cite{zhang2021taming}. To the best of our knowledge, there is no existing literature achieving such a favorable sample complexity. This  confirms the superiority of our method. The detailed comparison between our method and existing methods is demonstrated in Table~\ref{sequence}.
At last, we leverage our method to optimize the decentralized AUC maximization problem and the experimental results  confirm the superior  performance of our method.  Finally, we summarize the contribution of our work in the following. 
\begin{itemize}
	\item We developed a novel decentralized stochastic gradient descent ascent method for optimizing  finite-sum nonconvex-strongly-concave  minimization problems without periodically computing the full gradient as existing methods \cite{luo2020stochastic,zhang2021taming}. Therefore, our method is  efficient in computation. 
	\item We established the convergence rate for our proposed decentralized optimization method, which demonstrates that our method can achieve a better sample complexity than existing decentralized minimax optimization methods.  This is the first work achieving such a theoretical result for the decentralized minimax problem.
	\item We conducted extensive experiments on the AUC maximization problem, which confirms the effectiveness of our method in practical applications. 
\end{itemize}

\begin{table*}[h] 
	\setlength{\tabcolsep}{12pt}
	\begin{center}
		\begin{tabular}{l|cc|c}
			\toprule
			{Methods} &   {Sample Complexity} &  {Communication Complexity}  & Category  \\
			\hline
			DM-HSGD \cite{xian2021faster} & $O(\frac{\kappa^3}{(1-\lambda)^2\epsilon^3})$   &  $O(\frac{\kappa^3}{(1-\lambda)^2\epsilon^3})$  & Stochastic\\
			GT-SRVR \cite{zhang2021taming} & $O(n+\frac{\sqrt{n}\kappa^3}{(1-\lambda)^2\epsilon^2})$   & $O(\frac{\kappa^3}{(1-\lambda)^2\epsilon^2})$ & Finite-sum\\
			\hline
			Ours & $O(\frac{\sqrt{n}\kappa^3}{(1-\lambda)^2\epsilon^2})$   & $O(\frac{\kappa^3}{(1-\lambda)^2\epsilon^2})$ & Finite-sum\\
			\bottomrule 
		\end{tabular}
	\end{center}
	\caption{The comparison in sample and communication complexities between our method and baseline methods. Here, $\kappa$ is the condition number, $1-\lambda$ is the spectral gap, and $n$ is the number of samples on each worker.  \cite{zhang2021taming} proposed a variant, GT-SRVRI, by introducing an additional assumption:  $ \mathbb{E}[\|\nabla_{\mathbf{x}} f_i^{(k)}(\mathbf{x}, \mathbf{y})-\nabla_{\mathbf{x}} f^{(k)}(\mathbf{x}, \mathbf{y})\|^2]\leq \sigma^2$ and $ \mathbb{E}[\|\nabla_{\mathbf{y}} f_i^{(k)}(\mathbf{x}, \mathbf{y})-\nabla_{\mathbf{y}} f^{(k)}(\mathbf{x}, \mathbf{y})\|^2]\leq \sigma^2$, where $\sigma>0$,  $k\in[K]$, and $i\in[n]$. However, this is not a common assumption for the finite-sum optimization problem. Therefore, we do not compare it. 
	}
	\label{sequence}
\end{table*}

\section{Related Work}

\subsection{Minimax Optimization}

Minimax optimization has attracted a surge of attention in the machine  learning community in the past few years due to its widespread application in many machine learning models. To this end, a line of research is to develop efficient optimization methods \cite{sanjabi2018solving,nouiehed2019solving,jin2020local,yan2020optimal,zhang2021complexity,chen2020efficient,yang2020global,tran2020hybrid} to solve the minimax optimization problem. In particular, under the \textit{stochastic setting}, \cite{lin2020gradient} developed a single-loop stochastic gradient descent ascent (SGDA) method, which updates $\mathbf{x}$ and $\mathbf{y}$  for only one step with stochastic gradients in each iteration.  The sample complexity of SGDA for the nonconvex-strongly-concave minimax problem is $O(\kappa^3/\epsilon^4)$.  Later, \cite{qiu2020single,guo2021stochastic} combined the momentum technique with SGDA to accelerate the empirical convergence speed.  Moreover,  \cite{qiu2020single,huang2020accelerated} utilized the variance reduction technique STORM \cite{cutkosky2019momentum} to accelerate the convergence speed for nonconvex-strongly-concave minimax problems. 

As for the \textit{finite-sum setting}, \cite{luo2020stochastic} proposed a double-loop (SREDA) method, which updates $\mathbf{x}$ for one step with the variance-reduced gradient estimator SPIDER \cite{fang2018spider} and solves the maximization problem about $\mathbf{y}$ with multiple gradient ascent steps. As such, it can achieve the  $O(n+n^{1/2}\kappa^2/\epsilon^2)$ sample complexity for the \textit{finite-sum} nonconvex-strongly-concave minimax problem.  However, SREDA requires to periodically compute the full gradient, which is not practical for large-scale real-world applications.  In addition, its step size should be as small as $\epsilon$, which also limits its application for real-world tasks.  Recently, \cite{xu2020enhanced} resorted to the SpiderBoost \cite{wang2019spiderboost} variance reduction technique to tolerate a large step size.  But it still needs to compute the full gradient periodically so that it has the same sample complexity with SREDA.

\subsection{Decentralized Optimization}
In recent years, decentralized optimization methods have been applied to optimize large-scale machine learning models. In particular, \cite{lian2017can} proposed a decentralized stochastic gradient descent (DSGD) method based on the gossip communication scheme, while \cite{pu2020distributed,lu2019gnsd} used the gradient tracking communication scheme for DSGD. \cite{yu2019linear} applied the momentum technique to DSGD to accelerate the convergence speed. Afterwards, the variance reduction technique has been utilized to further accelerate the convergence speed of DSGD. For example, \cite{sun2020improving} combines SPIDER \cite{fang2018spider} with the gradient-tracking-based DSGD, achieving the near-optimal sample and communication complexity.  Besides, there are some works focusing on the communication-efficient methods by compressing gradients \cite{koloskova2019decentralized} or skipping communication rounds \cite{li2019communication}.  However, all these methods are designed for the minimization problem. Hence, they are not applicable to optimize Eq.~(\ref{loss}). 

A few efforts have been made to optimizing the decentralized minimax problem in the past two years. For example, \cite{liu2019decentralized} developed a  decentralized optimistic stochastic gradient method to train the nonconvex-nonconcave generative adversarial nets \cite{goodfellow2014generative}.  \cite{rogozin2021decentralized} focused on the strongly-convex-strongly-concave problem. Recently, \cite{beznosikov2021decentralized} proposed a communication-efficient method based on the stochastic extragradient algorithm. \cite{xian2021faster} developed a decentralized stochastic gradient descent ascent method based on the STORM \cite{cutkosky2019momentum} gradient estimator for the stochastic minimax problem, rather than the finite-sum problem. \cite{zhang2021taming} proposed GT-SRVR based on the SPIDER gradient estimator \cite{zhang2021taming} for  finite-sum problems,  which requires to periodically compute the full gradient.

\section{Efficient Decentralized Stochastic Gradient Descent Ascent Method}

\begin{algorithm} [!h]
	\caption{Efficient Decentralized Stochastic Gradient Descent Ascent  (DSGDA)}
	\label{dsgda}
	\begin{algorithmic}[1]
		\REQUIRE $\mathbf{x}_{0}^{(k)}=\mathbf{x}_{-1}^{(k)}=\mathbf{x}_{0}$, $\mathbf{y}_{0}^{(k)}=\mathbf{y}_{-1}^{(k)}=\mathbf{y}_{0}$, $\mathbf{v}_{-1}^{(k)}=\mathbf{a}_{-1}^{(k)}=0$, $\mathbf{u}_{-1}^{(k)}=\mathbf{b}_{-1}^{(k)}=0$,\\ $\mathbf{g}_{i,-1}^{(k)}=0$, $\mathbf{h}_{i,-1}^{(k)}=0$ for $i\in \{1,2,\cdots n\}$.
		\FOR{$t=0,\cdots, T-1$} 
		
		\STATE Randomly select  samples $\mathcal{S}_{t}$ with $|\mathcal{S}_{t}|=s_t$ and then compute $\mathbf{v}_t^{(k)}$ and $\mathbf{u}_t^{(k)}$ as Eq.~(\ref{vt})  and Eq.~(\ref{ut})

		\STATE Update $\mathbf{x}$:
		
		$\mathbf{a}_t^{(k)}  = \sum_{j\in \mathcal{N}_k}w_{kj}\mathbf{a}_{t-1}^{(j)} +  \mathbf{v}_t^{(k)} - \mathbf{v}_{t-1}^{(k)} $
		
		$\mathbf{x}_{t+\frac{1}{2}}^{(k)} =  \sum_{j\in \mathcal{N}_k} w_{kj}\mathbf{x}_t^{(j)}-\gamma_1 \mathbf{a}_t^{(k)}$
		
		$\mathbf{x}_{t+1}^{(k)}  =  \mathbf{x}_t^{(k)} +\eta ({\mathbf{x}}_{t+\frac{1}{2}}^{(k)}   - \mathbf{x}_t^{(k)} ) $
		
		\STATE Update  $\mathbf{y}$:
		
		$\mathbf{b}_t^{(k)}  = \sum_{j\in \mathcal{N}_k} w_{kj}\mathbf{b}_{t-1}^{(j)} +  \mathbf{u}_t^{(k)} - \mathbf{u}_{t-1}^{(k)} $

		$\mathbf{y}_{t+\frac{1}{2}}^{(k)} =  \sum_{j\in \mathcal{N}_k} w_{kj}\mathbf{y}_t^{(j)}+\gamma_2 \mathbf{b}_t^{(k)}$
		
		$\mathbf{y}_{t+1}^{(k)}  =  \mathbf{y}_t^{(k)} +\eta  ({\mathbf{y}}_{t+\frac{1}{2}}^{(k)}   - \mathbf{y}_t^{(k)} ) $
		
		\STATE Update $\mathbf{g}$ and $\mathbf{h}$:
		
		$\mathbf{g}_{i, t}^{(k)} = \begin{cases}
			\nabla_{\mathbf{x}} f^{(k)}_{i}(\mathbf{x}_{t}^{(k)} , \mathbf{y}_{t}^{(k)} ),  & \text { for } i \in  \mathcal{S}_{t}\\
			\mathbf{g}_{i, t-1}^{(k)} , & \text{otherwise} \\
		\end{cases} $

		$\mathbf{h}_{i, t}^{(k)} = \begin{cases}
			\nabla_{\mathbf{y}} f^{(k)}_{i}(\mathbf{x}_{t}^{(k)} , \mathbf{y}_{t}^{(k)} ),  & \text { for } i \in  \mathcal{S}_{t}\\
			\mathbf{h}_{i, t-1}^{(k)} , & \text{otherwise} \\
		\end{cases} $

		\ENDFOR
	\end{algorithmic}
\end{algorithm}

\subsection{Problem Setup}
In this paper, the communication network  in the decentralized training system is represented by $\mathcal{G}=\{{P}, {W}\}$. Here, ${P}=\{p_1, p_2, \cdots, p_K\}$ represents $K$ workers.  ${W}=[w_{ij}]\in \mathbb{R}^{K\times K}$ is the adjacency matrix, denoting the connection among these $K$ workers. When $w_{ij}>0$, the workers $p_i$ and $p_j$ are connected.  Otherwise, they are disconnected and then cannot communicate to each other.   In addition, for the adjacency matrix, we have the following assumption. 

\begin{assumption} \label{graph}
	The adjacency matrix $W$  satisfies following properties:
	\begin{itemize}
		\item $W$ is nonnegative, i.e., $w_{ij}\geq 0$. 
		\item $W$  is symmetric, i.e., $W^T=W$.
		\item $W$  is  doubly stochastic, i.e., $W\mathbf{1}=\mathbf{1}$ and $\mathbf{1}^TW=\mathbf{1}^T$. 
		\item The  eigenvalues  $\{\lambda_i\}_{i=1}^{n}$ of $W$ satisfy $|\lambda_n|\leq \cdots \leq |\lambda_2|< |\lambda_1|=1$. 
	\end{itemize}
\end{assumption}
This assumption is also used in existing works \cite{lian2017can,koloskova2019decentralized,liu2019decentralized}. In this paper, the spectral gap is represented by $1-\lambda$ where $\lambda\triangleq |\lambda_2|$.

\subsection{Method}
In Algorithm~\ref{dsgda}, we developed a novel efficient descentralized stochastic gradient descent ascent (DSGDA) method.  Specifically, each worker computes the stochastic gradient with its local dataset and then updates its local model parameters. In detail, at the $t$-th iteration, the $k$-th worker samples a mini-batch of samples $\mathcal{S}_{t}$ to compute the variance-reduced gradient regarding $\mathbf{x}$ as follows:
\begin{equation} \label{vt}
	\begin{aligned}
		\mathbf{v}_t^{(k)}& =(1-\rho_{t}) \mathbf{v}_{t-1}^{(k)} +  \frac{1}{s_t} \sum_{i \in \mathcal{S}_{t}}(\nabla_{\mathbf{x}} f^{(k)}_{i}(\mathbf{x}_{t}^{(k)}, \mathbf{y}_{t}^{(k)})-\nabla_{\mathbf{x}} f^{(k)}_{i}(\mathbf{x}_{t-1}^{(k)}, \mathbf{y}_{t-1}^{(k)}))  \\
		& +\rho_{t}\Big(\frac{1}{s_t} \sum_{i \in \mathcal{S}_{t}}(\nabla_{\mathbf{x}} f^{(k)}_{i}(\mathbf{x}_{t-1}^{(k)}, \mathbf{y}_{t-1}^{(k)})-\mathbf{g}_{i, t-1}^{(k)}) +\frac{1}{n} \sum_{j=1}^{n} \mathbf{g}_{j, t-1}^{(k)}\Big) \ , 
	\end{aligned}
\end{equation}
where $ \rho_{t} \in [0, 1]$ is a hyperparameter, $\mathbf{x}_{t}^{(k)}$ and $\mathbf{y}_{t}^{(k)}$ denote the model parameters on the $k$-th worker in the $t$-th iteration, $\nabla_{\mathbf{x}} f^{(k)}_{i}(\mathbf{x}_{t}^{(k)}, \mathbf{y}_{t}^{(k)})$  denotes the stochastic gradient regarding $\mathbf{x}$,  $\mathbf{v}_t^{(k)}$ is the corresponding  variance-reduced gradient, $\mathbf{g}_{i, t}^{(k)}$ stores the stochastic gradient of the $i$-th sample on the $k$-th worker, which is updated as follows:
\begin{equation}
	\mathbf{g}_{i, t}^{(k)} = \begin{cases}
		\nabla_{\mathbf{x}} f^{(k)}_{i}(\mathbf{x}_{t}^{(k)} , \mathbf{y}_{t}^{(k)} ),  & \text { for } i \in  \mathcal{S}_{t}\\
		\mathbf{g}_{i, t-1}^{(k)} , & \text{otherwise} \ . \\
	\end{cases} 
\end{equation}

Similarly, to update $\mathbf{y}$, the $k$-th worker uses the same mini-batch of samples $\mathcal{S}_{t}$ to compute the variance-reduced gradient regarding $\mathbf{y}$ as follows:
\begin{equation} \label{ut}
	\begin{aligned}
		\mathbf{u}_t^{(k)} & =(1-\rho_{t})\mathbf{u}_{t-1}^{(k)} + \frac{1}{s_t} \sum_{i \in \mathcal{S}_{t}}(\nabla_{\mathbf{y}} f^{(k)}_{i}(\mathbf{x}_{t}^{(k)}, \mathbf{y}_{t}^{(k)})-\nabla_{\mathbf{y}} f^{(k)}_{i}(\mathbf{x}_{t-1}^{(k)}, \mathbf{y}_{t-1}^{(k)}))\\
		&  +\rho_{t}\Big(\frac{1}{s_t} \sum_{i \in \mathcal{S}_{t}}(\nabla_{\mathbf{y}} f^{(k)}_{i}(\mathbf{x}_{t-1}^{(k)}, \mathbf{y}_{t-1}^{(k)})-\mathbf{h}_{i, t-1}^{(k)})  +\frac{1}{n} \sum_{j=1}^{n} \mathbf{h}_{j, t-1}^{(k)}\Big) \ , 
	\end{aligned}
\end{equation}
where $\mathbf{u}_t^{(k)}$ is the variance-reduced gradient for the variable $\mathbf{y}$, $\mathbf{h}_{i, t}^{(k)}$ stores the stochastic gradient of the $i$-th sample on the $k$-th worker for the variable $\mathbf{y}$. Similar to $\mathbf{g}_{i, t}^{(k)}$ , $\mathbf{h}_{i, t}^{(k)}$ is updated as follows:
\begin{equation}
	\mathbf{h}_{i, t}^{(k)} = \begin{cases}
		\nabla_{\mathbf{y}} f^{(k)}_{i}(\mathbf{x}_{t}^{(k)} , \mathbf{y}_{t}^{(k)} ),  & \text { for } i \in  \mathcal{S}_{t}\\
		\mathbf{h}_{i, t-1}^{(k)} , & \text{otherwise}   \ . \\
	\end{cases} 
\end{equation}

After obtaining the variance-reduced gradient,  the $k$-th worker employs the gradient tracking communication scheme to communicate  with its neighboring workers:
\begin{equation}
	\begin{aligned}
		& \mathbf{a}_t^{(k)}  = \sum_{j\in \mathcal{N}_k}w_{kj}\mathbf{a}_{t-1}^{(j)} +  \mathbf{v}_t^{(k)} - \mathbf{v}_{t-1}^{(k)} \ ,   \mathbf{b}_t^{(k)}  = \sum_{j\in \mathcal{N}_k}w_{kj}\mathbf{b}_{t-1}^{(j)} +  \mathbf{u}_t^{(k)} - \mathbf{u}_{t-1}^{(k)}  \ , \\
	\end{aligned}
\end{equation}
where  $\mathcal{N}_k$ is the neighboring workers of the $k$-th worker,  $\mathbf{a}_t^{(k)}$  and $\mathbf{b}_t^{(k)}$ are the gradients after communicating with the neighboring workers. 

Then, the $k$-th worker updates its local model parameter $\mathbf{x}_{t}^{(k)}$ as follows:
\begin{equation}
	\begin{aligned}
		& \mathbf{x}_{t+\frac{1}{2}}^{(k)} =  \sum_{j\in \mathcal{N}_k}w_{kj}\mathbf{x}_t^{(j)}-\gamma_1 \mathbf{a}_t^{(k)}  \ ,  \quad \mathbf{x}_{t+1}^{(k)}  =  \mathbf{x}_t^{(k)} +\eta ({\mathbf{x}}_{t+\frac{1}{2}}^{(k)}   - \mathbf{x}_t^{(k)} )  \ ,
	\end{aligned}
\end{equation}
where $\gamma_1>0$ and $0<\eta<1$ are two hyperparameters.  Similarly, $\mathbf{y}_{t}^{(k)}$ is also updated as follows:
\begin{equation}
	\begin{aligned}
		& {\mathbf{y}}_{t+\frac{1}{2}}^{(k)}  = \sum_{j\in \mathcal{N}_k} w_{kj}\mathbf{y}_t^{(j)}+\gamma_2 \mathbf{b}_t ^{(k)} \ ,  \quad  \mathbf{y}_{t+1}^{(k)}  =  \mathbf{y}_t^{(k)} +\eta ({\mathbf{y}}_{t+\frac{1}{2}}^{(k)}   - \mathbf{y}_t^{(k)} )  \ ,
	\end{aligned}
\end{equation}
where $\gamma_2>0$ and $0<\eta<1$ are two hyperparameters.

All workers in the decentralized training system repeat the aforementioned steps to update $\mathbf{x}$ and $\mathbf{y}$ until it converges.

The variance-reduced gradient estimator  in Eq.~(\ref{vt}) was first proposed in \cite{li2021zerosarah}. But they only focus on the minimization problem and ignore the decentralized setting.  In fact, it is nontrivial to apply this variance-reduced gradient estimator to the decentralized minimax problem. Especially, it is challenging to establish the convergence rate, which is shown in the next section.

\section{Theoretical Analysis}

\subsection{Convergence Rate}

To establish the convergence rate of Algorithm~\ref{dsgda}, we introduce the following assumptions, which are commonly used in existing works.

\begin{assumption} \label{assumption_smooth}
	(Smoothness) Each function $f_i^{(k)}(\cdot, \cdot)$ is $L$-smooth. i.e., for any $ (\mathbf{x}_1, \mathbf{y_1})$ and $ (\mathbf{x}_2, \mathbf{y_2})$,  there exists $L>0$, such that
	\begin{equation}
		\begin{aligned}
			& \|\nabla f_i^{(k)}(\mathbf{x}_1, \mathbf{y}_1)-\nabla f_i^{(k)}(\mathbf{x}_2, \mathbf{y}_2)\|^2  \leq  L^2\|\mathbf{x}_1 - \mathbf{x}_2\|^2 + L^2\|\mathbf{y}_1 - \mathbf{y}_2\|^2  \ .\\
		\end{aligned}
	\end{equation}
\end{assumption}

\begin{assumption} \label{assumption_strong}
	(Strong concavity) The function $f^{(k)}(\mathbf{x}, \mathbf{y})$ is $\mu$-strongly concave with respect to $\mathbf{y}$, i.e., for any $ (\mathbf{x}, \mathbf{y_1})$ and $ (\mathbf{x}, \mathbf{y_2})$,  there exists $\mu>0$, such that
	\begin{equation}
		\begin{aligned}
			& f^{(k)}(\mathbf{x}, \mathbf{y_1})\leq f^{(k)}(\mathbf{x}, \mathbf{y_2}) + \langle \nabla_{\mathbf{y}} f^{(k)}(\mathbf{x}, \mathbf{y_2}), \mathbf{y_1} - \mathbf{y_2}\rangle  - \frac{\mu}{2} \|\mathbf{y_1} - \mathbf{y_2}\|^2 \ .
		\end{aligned}
	\end{equation}
\end{assumption}

Based on these assumption, we denote the condition number by $\kappa=L/\mu>1$.  In addition, we denote $\mathbf{y}^*({\mathbf{x}})=\arg\max_{\mathbf{y}} \frac{1}{K}\sum_{k=1}^{K}f^{(k)}({\mathbf{x}}, \mathbf{y})$ and 
$\Phi(\mathbf{x})=\frac{1}{K}\sum_{k=1}^{K}\Phi^{(k)}(\mathbf{x})=\frac{1}{K}\sum_{k=1}^{K}f^{(k)}(\mathbf{x}, \mathbf{y}^{*}({\mathbf{x}}))$.  According to \cite{lin2020gradient}, we can know that  $\Phi(\mathbf{x})$ is $L_{\Phi}$-smooth, where $L_{\Phi}=2\kappa L$. 

Throughout this paper, we denote $\bar{\mathbf{c}}=\frac{1}{K}\sum_{k=1}^{K}\mathbf{c}_k$, where $\mathbf{c}_k$ is the variable on the $k$-th worker. Then, we establish the convergence rate of our method for nonconvex-strongly-concave problems in Theorem~\ref{theorem}.
\begin{theorem} \label{theorem}
	Given Assumptions~\ref{graph}-\ref{assumption_strong}, if setting $s_t=s_1$ for $t>0$,  $\rho_t=\rho_1=\frac{s_1}{2n}$ for $t>0$, $\rho_0=1$, and 
	\begin{align}
		& \gamma_1\leq \min\left\{\frac{(1-\lambda)^2}{760\kappa L }, \frac{1}{150\kappa L} ,  \frac{\gamma_2}{15\kappa^2 } \right\} \ , \quad  \gamma_2 \leq \min\left\{ \frac{(1-\lambda)^2}{760 L },  \frac{  1}{624  \kappa L}, \frac{1}{6L}   \right\}  \ ,  \quad \eta< \min \left\{1, \frac{1}{2\gamma_1 L_{\Phi}}\right\} \ , 
	\end{align}
	our algorithm is able to achieve the following convergence rate:
	\begin{align}
		&  \frac{1}{T}\sum_{t=0}^{T-1}\left(\mathbb{E}[\|\nabla \Phi(\bar{\mathbf{x}}_{t})\|^2]   +L^2  \mathbb{E}[\|\bar{\mathbf{y}}_t -  {\mathbf{y}^*(\overline{\mathbf{x}}_t)}\|^2]\right)\leq \frac{2(\Phi({\mathbf{x}}_{0})- \Phi(\mathbf{x}_*))}{\gamma_1\eta T}  +  \frac{12 L^2}{\gamma_2\eta T \mu} \mathbb{E}[\|\bar{\mathbf{y}}_{0}   - \mathbf{y}^{*}(\bar{\mathbf{x}}_0)\| ^2 ]  \notag \\
		& \quad + \frac{30}{ T} \frac{n-s_0}{s_0s_1} \frac{1}{K}\sum_{k=1}^{K}\frac{1}{n}\sum_{i=1}^{n}  \mathbb{E}[\|\nabla_{\mathbf{x}} f^{(k)}_i(\mathbf{x}_{0}^{(k)}, \mathbf{y}_{0}^{(k)})\|^2]  +\frac{556\kappa^2}{ T} \frac{n-s_0}{s_0s_1} \frac{1}{K}\sum_{k=1}^{K}\frac{1}{n}\sum_{i=1}^{n} \mathbb{E}[\|\nabla_{\mathbf{y}} f^{(k)}_{i}(\mathbf{x}_{0}^{(k)}, \mathbf{y}_{0}^{(k)})\|^{2}]  \  ,
	\end{align}
	where $\mathbf{x}_*$ denotes the optimal solution. 
\end{theorem}

\begin{corollary} \label{corollary}
	Given Assumptions~\ref{graph}-\ref{assumption_strong}, by setting $s_0=s_1=\sqrt{n}$, we can get $\gamma_1=O((1-\lambda)^2/\kappa^3)$,   $\gamma_2=O((1-\lambda)^2/\kappa)$, and $\eta=O(1)$ under the worst case. Then,  by setting  $T=O(\frac{\kappa^3}{(1-\lambda)^2\epsilon^2})$,  our algorithm can achieve the $\epsilon$-accuracy solution:
	\begin{equation}
		\begin{aligned}
			& \quad \frac{1}{T}\sum_{t=0}^{T-1} (\mathbb{E}[\|\nabla \Phi(\bar{\mathbf{x}}_{t})\|^2]   + L^2\mathbb{E}[\|\bar{\mathbf{y}}_t -  {\mathbf{y}^*(\overline{\mathbf{x}}_t})\|^2]) \leq \epsilon^2 \ .\\
		\end{aligned}
	\end{equation}

\end{corollary}

\begin{remark} \label{remark_lr_our}
	From Corollary~\ref{corollary}, it is easy to know that  the communication complexity of our method is $O(\frac{\kappa^3}{(1-\lambda)^2\epsilon^2})$ and the sample complexity is $T\times \sqrt{n} = O(\frac{\sqrt{n}\kappa^3}{(1-\lambda)^2\epsilon^2})$. 
	
\end{remark}

\begin{remark} \label{remark_lr_storm}
	Compared with \cite{xian2021faster} whose step sizes are $O((1-\lambda)^2\epsilon/\kappa^3)$ and $O((1-\lambda)^2\epsilon/\kappa)$, our step sizes, i.e., $\gamma_1$, $\gamma_2$, and $\eta$,  are independent of $\epsilon$. In addition, our communication complexity is better than $O(\frac{\kappa^3}{(1-\lambda)^2\epsilon^3})$ of \cite{xian2021faster}.  
\end{remark}

\begin{remark} \label{remark_lr_spider}
	The step sizes of \cite{zhang2021taming} are also independent of $\epsilon$ and have the same order dependence on the spectral gap and condition number as our $\gamma_1$ and $\gamma_2$. But its sample complexity $O(n+\frac{\sqrt{n}\kappa^3}{(1-\lambda)^2\epsilon^2})$ is worse than ours because it needs to periodically compute the full gradient.  
\end{remark}

\subsection{Proof Sketch}
In this subsection, we present the proof sketch of our Theorem~\ref{theorem}. The detailed proof can be found in supplementary materials. 

To investigate the convergence rate of our method, we propose a novel potential function as follows:
\begin{align} \label{eq:potential_function}
	& H_{t}=\mathbb{E}[\Phi(\bar{\mathbf{x}}_{t})]+ C_0 \mathbb{E}[\|\bar{\mathbf{y}}_{t}   - \mathbf{y}^{*}(\bar{\mathbf{x}}_t)\| ^2 ]  \notag \\
	& \quad +  \frac{C_1}{K}\sum_{k=1}^{K}\mathbb{E}[\|\nabla_{\mathbf{x}} f^{(k)}(\mathbf{x}^{(k)}_t, \mathbf{y}^{(k)}_t)-\mathbf{v}_t^{(k)}\|^2] +  \frac{C_2}{K}\sum_{k=1}^{K}\mathbb{E}[\|\nabla_{\mathbf{y}} f^{(k)}(\mathbf{x}^{(k)}_t, \mathbf{y}^{(k)}_t)-\mathbf{u}_t^{(k)}\|^2] \notag \\
	& \quad + \frac{C_3}{K} \sum_{k=1}^{K}\mathbb{E}[\frac{1}{n} \sum_{j=1}^{n}\|\nabla_{\mathbf{x}} f^{(k)}_{j}(\mathbf{x}_{t}^{(k)}, \mathbf{y}_{t}^{(k)})-\mathbf{g}_{j,t}^{(k)}\|^{2}]  + \frac{C_4}{K}\sum_{k=1}^{K}\mathbb{E}[\frac{1}{n} \sum_{j=1}^{n}\|\nabla_{\mathbf{y}} f^{(k)}_{j}(\mathbf{x}_{t}^{(k)}, \mathbf{y}_{t}^{(k)})-\mathbf{h}_{j,t}^{(k)}\|^{2}] \notag \\
	& \quad + \frac{C_5}{K}\sum_{k=1}^{K} \mathbb{E}[\|\bar{\mathbf{x}}_t-\mathbf{x}^{(k)}_{t}\|^2]+ \frac{C_6}{K}\sum_{k=1}^{K} \mathbb{E}[\|\bar{\mathbf{y}}_t-\mathbf{y}^{(k)}_{t}\|^2] \notag \\
	& \quad  + \frac{C_7}{K}\sum_{k=1}^{K} \mathbb{E}[\|\bar{\mathbf{a}}_t-\mathbf{a}^{(k)}_{t}\|^2]+ \frac{C_8}{K}\sum_{k=1}^{K}\mathbb{E}[ \|\bar{\mathbf{b}}_t-\mathbf{b}^{(k)}_{t}\|^2 ] \  , 
\end{align}
where 
\begin{align} 
	&  C_0 = \frac{6\gamma_1  L^2}{\gamma_2\mu}  \ ,   \quad  C_1 =  \frac{3\gamma_1\eta}{\rho_{1}}  \ ,    \quad C_2 = \frac{51\eta \gamma_1 L^2}{\rho_{1}\mu^2} \ ,  \quad  C_3 = \frac{ 14n\rho_{1}\gamma_1\eta}{s_{1}^2} \ , \quad  C_4 = \frac{226n\rho_{1}\eta \gamma_1 L^2}{s_1^2\mu^2}   \ ,  \notag \\
	& C_5 = \frac{22568\gamma_1 \kappa^2  L^2}{(1-\lambda^2)} \ , \quad C_6 = \frac{22568\gamma_1 \kappa^2  L^2}{(1-\lambda^2)} \ , \quad  C_7 = \frac{(1-\lambda^2)\gamma_1\eta}{6\rho_{1}}  \ , \quad   C_8 =  \frac{(1-\lambda^2)\eta \gamma_1 L^2}{6\rho_{1}\mu^2}   \ .
\end{align}

Then, we can investigate how the potential function evolves across iterations by studying each item in this potential function.  In particular, we can get 
\begin{equation}
	\begin{aligned}
		&\quad  H_{t+1} - H_{t}  \leq - \frac{\gamma_1\eta}{2} \mathbb{E}[\|\nabla \Phi(\bar{\mathbf{x}}_{t})\|^2] - \frac{\gamma_1\eta L^2}{2}   \mathbb{E}[\|\bar{\mathbf{y}}_t - {\mathbf{y}^*(\overline{\mathbf{x}}_t)}\|^2] \ . 
	\end{aligned}
\end{equation}
Based on this inequality, we can complete the proof. The detailed proof can be found in Supplementary Materials.

\section{Experiments}
In this section, we present  experimental results to demonstrate the empirical performance of our method. 

\subsection{AUC Maximization}
AUC maximization is a commonly used method for the imbalanced data classification problem. Recently, \cite{ying2016stochastic}  reformulated the AUC maximization problem as an minimax optimization problem to facilicate  stochastic training for large-scale data. In our experiment, we employ our method to optimize the AUC maximization problem, which is defined as follows:

\begin{equation}
	\begin{aligned}
		& \min_{\bm{\theta}, \hat{\theta}_1, \hat{\theta}_2}\max_{\tilde{\theta}}\ell(\bm{\theta}, \hat{\theta}_1, \hat{\theta}_2, \tilde{\theta}) = (1-p) (\bm{\theta}^T\mathbf{a} - \hat{\theta}_1)^2\mathbb{I}_{[b=1]} + p(\bm{\theta}^T\mathbf{a} -\hat{\theta}_2)^2\mathbb{I}_{[b=-1]} \\
		& + 2(1+\tilde{\theta})(p\bm{\theta}^T\mathbf{a}\mathbb{I}_{[b=-1]} - (1-p)\bm{\theta}^T\mathbf{a}\mathbb{I}_{[b=1]})  - p(1-p)\tilde{\theta}^2  + \rho \sum_{j=1}^{d}\frac{\bm{\theta}_j^2}{1+\bm{\theta}_j^2}  \ , 
	\end{aligned}
\end{equation}
where $\bm{\theta}\in \mathbb{R}^d$ denotes the model parameter of the classifier, $\hat{\theta}_1\in \mathbb{R}$, $\hat{\theta}_2\in \mathbb{R}$ are two auxiliary model parameters for the minimization subproblem, $\tilde{\theta}\in \mathbb{R}$ is the model parameter for the maximization subproblem, $\{\mathbf{a}, b\}$ denotes training samples, $\rho>0$ is a hyperparameter for the regularization term. In our experiments, we set $\rho$ to 0.001. Obviously, this is a nonconvex-strongly-concave optimization problem. Then, we can use our Algorithm~\ref{dsgda} to optimize this problem.

\subsection{Experimental Settings}

In our experiments, we employ three binary classification datasets: a9a, w8a, and ijcnn1. All of them are  imbalanced datasets.  The detailed information about these datasets can be found in LIBSVM \footnote{\url{https://www.csie.ntu.edu.tw/~cjlin/libsvmtools/datasets/}}. We randomly select $20\%$ samples as the testing set and the left samples as the training set.  Throughout our experiments, we employ ten workers. Then, the training samples are randomly distributed to all workers.

To evaluate the performance of our method, we compare it with the state-of-the-art decentralized optimization algorithm: DM-HSGD \cite{xian2021faster},  GT-SRVR \cite{zhang2021taming}, and GT-SRVRI \cite{zhang2021taming}.  As for their step sizes, we set them in terms of Remarks~\ref{remark_lr_storm} and~\ref{remark_lr_spider}. Specifically, since we employed the  Erdos-Renyi random graph with the edge probability being $0.5$ to generate the communication network, whose spectral gap is in the order of $O(1)$ \cite{ying2021exponential}, we assume the spectral gap as $0.5$. Additionally, the solution accuracy $\epsilon$ is set to $0.01$. Then, the step sizes of two variables of DM-HSGD are set to $(1-\lambda)^2\epsilon=0.5^2\times0.01$ and the coefficient for the variance-reduced gradient estimator is set to $\epsilon\min\{1, n\epsilon\}=0.01\times 10 \times 0.01$ according to Theorem~1 in \cite{xian2021faster}. As for GT-SRVR and GT-SRVRI, the step sizes of two variables are set to $0.5^2$ based on Remark \ref{remark_lr_spider} \footnote{Since GT-SRVR's theoretical step size leads to divergence for a9a dataset, we scaled it by $0.01$ for the random graph and $0.5$ for the line graph, respectively.  }. As for our method, the step sizes $\gamma_1$ and $\gamma_2$ are set to $(1-\lambda)^2=0.5^2$ in terms of Remark~\ref{remark_lr_our}. Note that we omit the condition number for all step sizes since they are the same for all methods and it is difficult to obtain. As for our method,  since $\eta$ is independent of the spectral gap and the solution accuracy, we set it to $0.9$ throughout our experiments. Additionally, $\rho_1$ is set to $0.5/\sqrt{n}$. Moreover, the batch size is set to $\sqrt{n}$ for GT-SRVR  and our method according to Corollary~\ref{corollary}, and that is set to  64 for DM-HSGD.

\begin{figure*}[ht]
	\centering 
	\subfigure[a9a]{
		\includegraphics[scale=0.387]{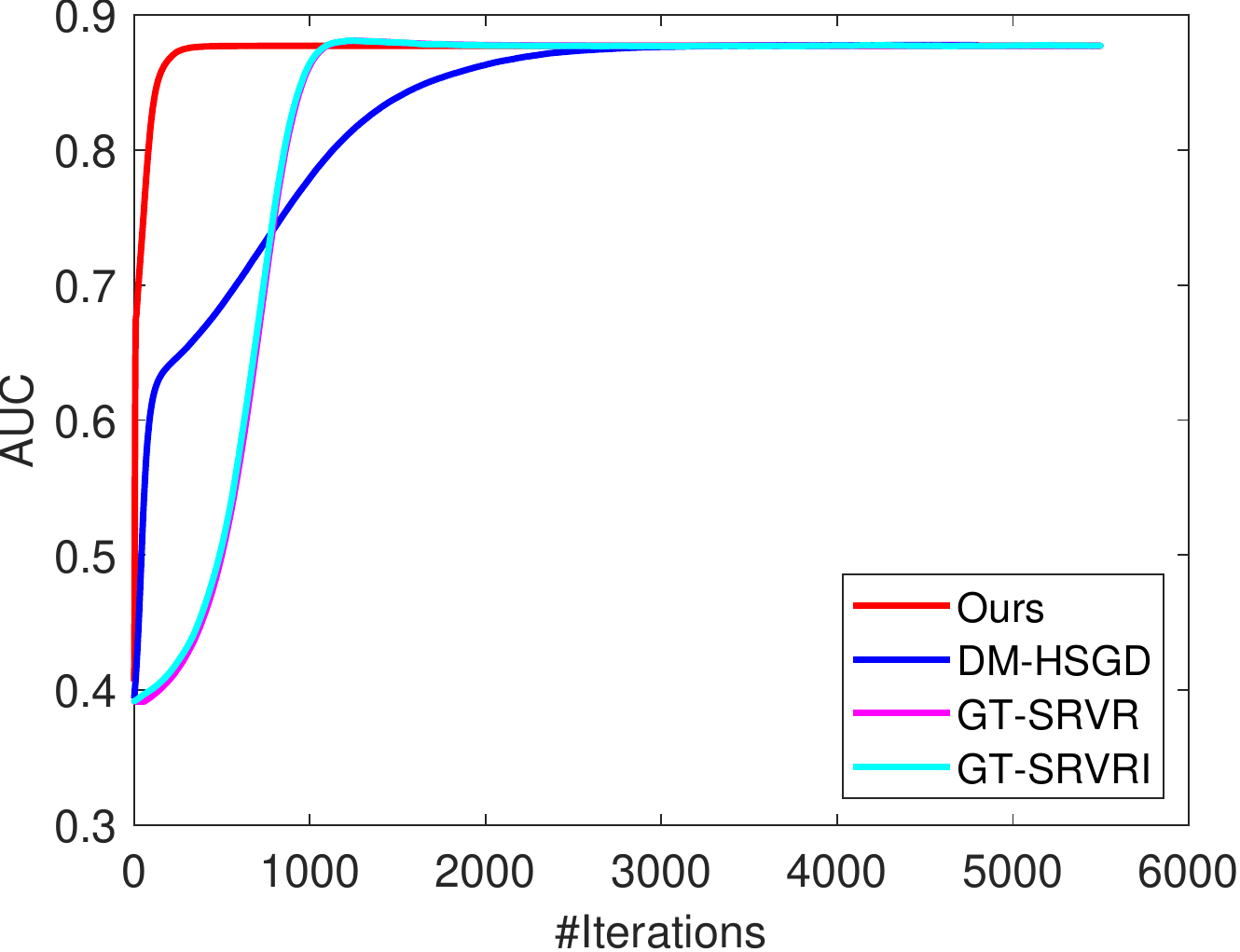}
	}
	\hspace{-12pt}
	\subfigure[w8a]{
		\includegraphics[scale=0.387]{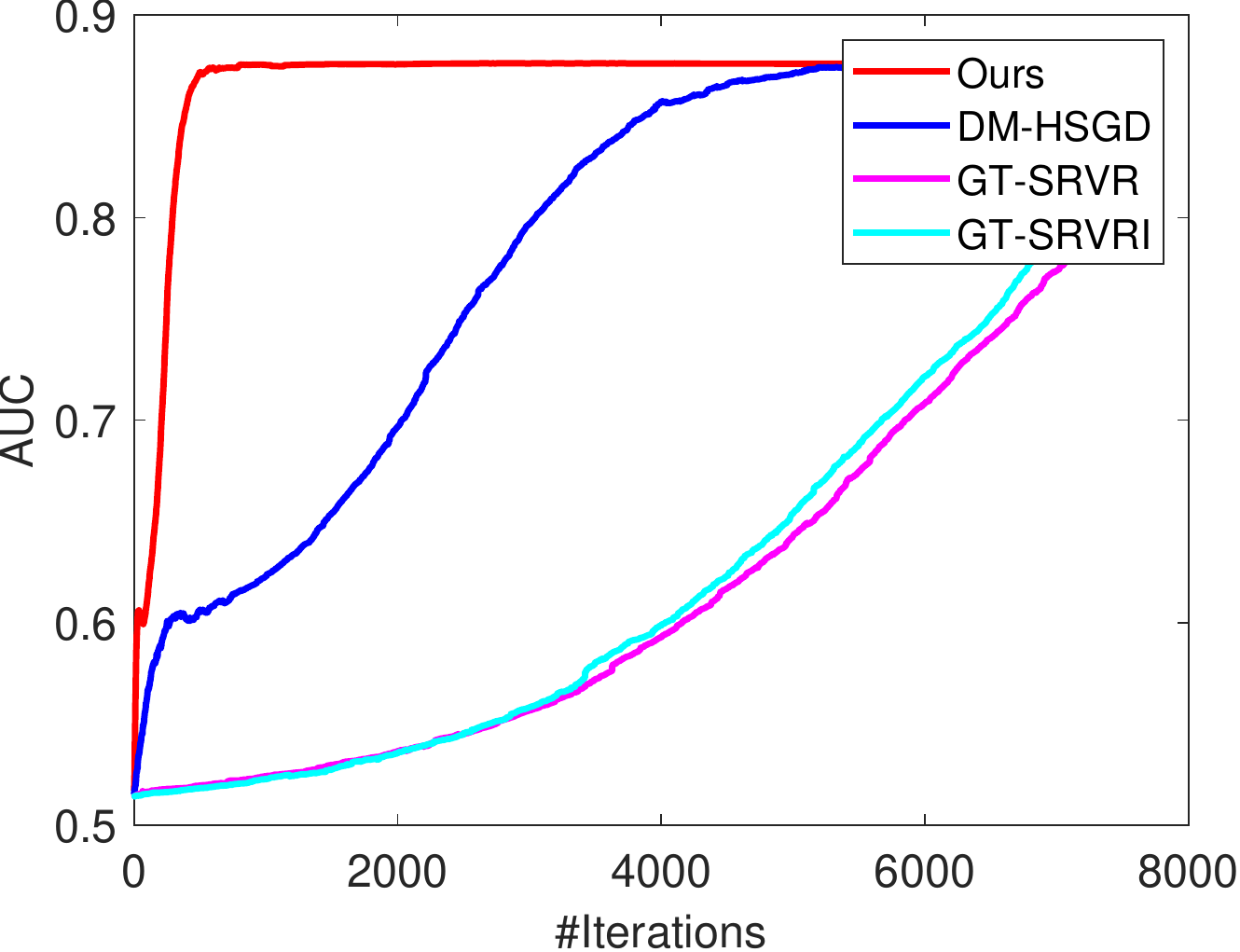}
	}
	\hspace{-12pt}
	\subfigure[ijcnn1]{
		\includegraphics[scale=0.387]{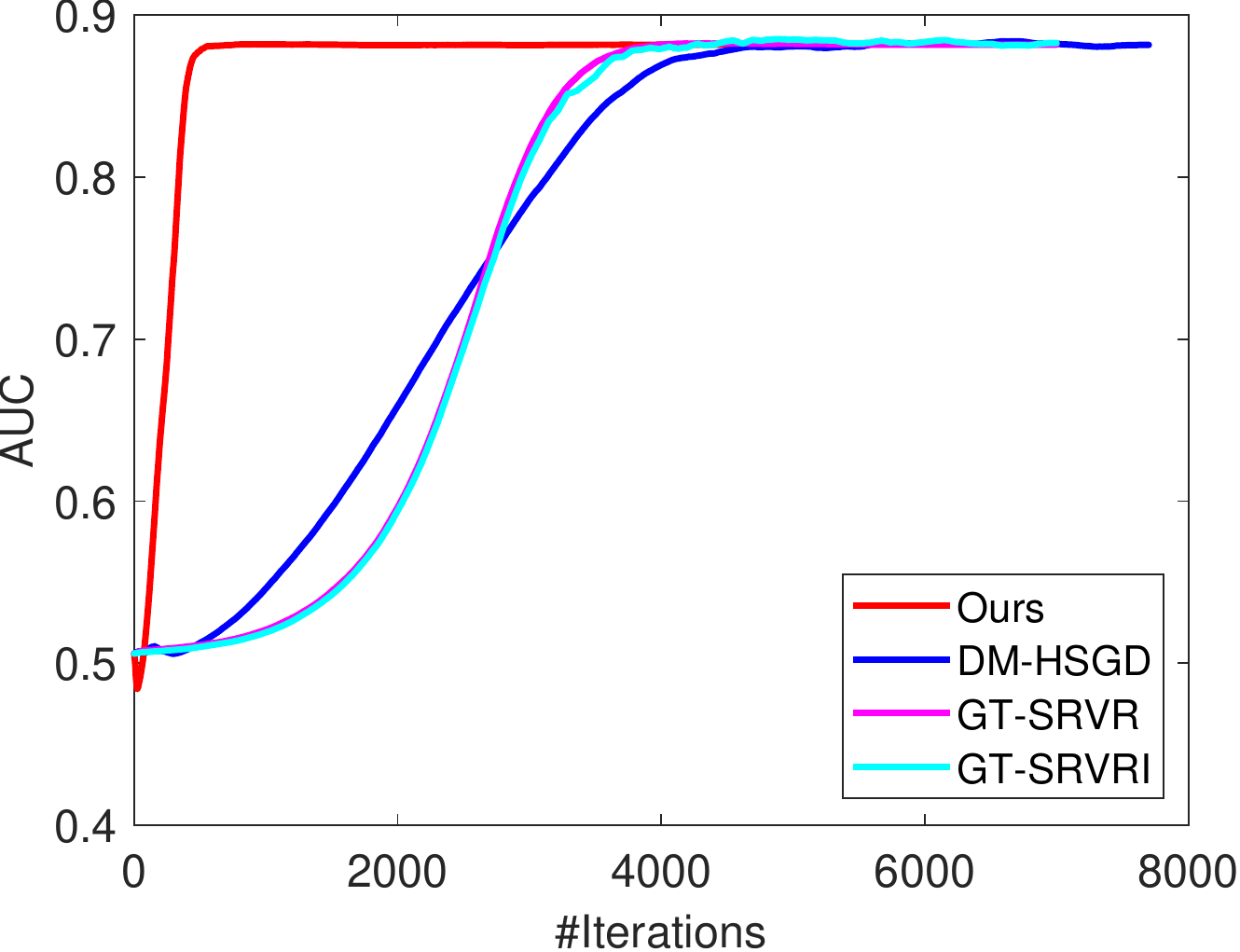}
	}
	\caption{The test AUC versus the number of iterations when using the random communication graph. }
	\label{fig_rand_iter}
\end{figure*}
\begin{figure*}[h]
	\centering 
	\subfigure[a9a]{
		\includegraphics[scale=0.4]{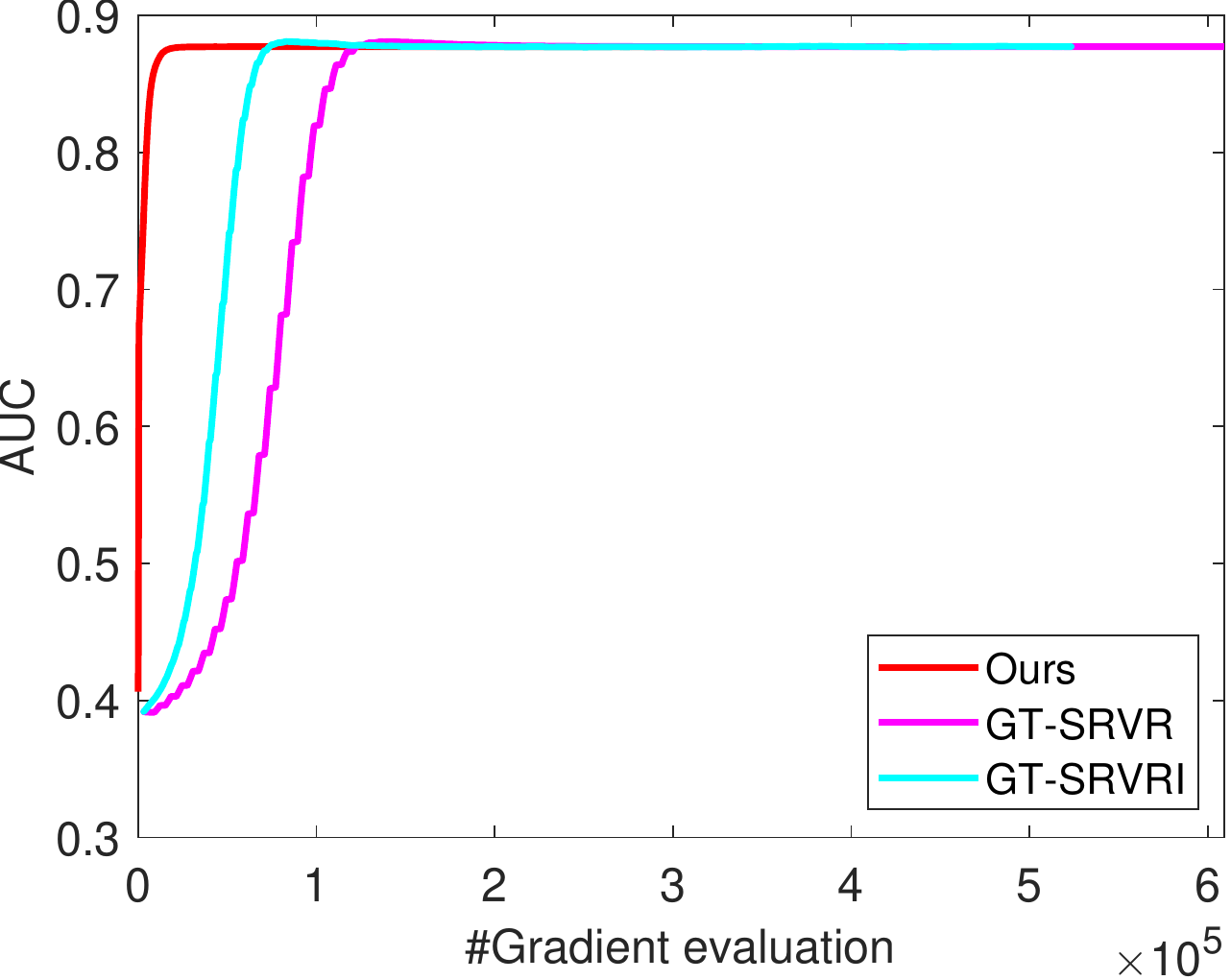}
	}
	\hspace{-10pt}
	\subfigure[w8a]{
		\includegraphics[scale=0.4]{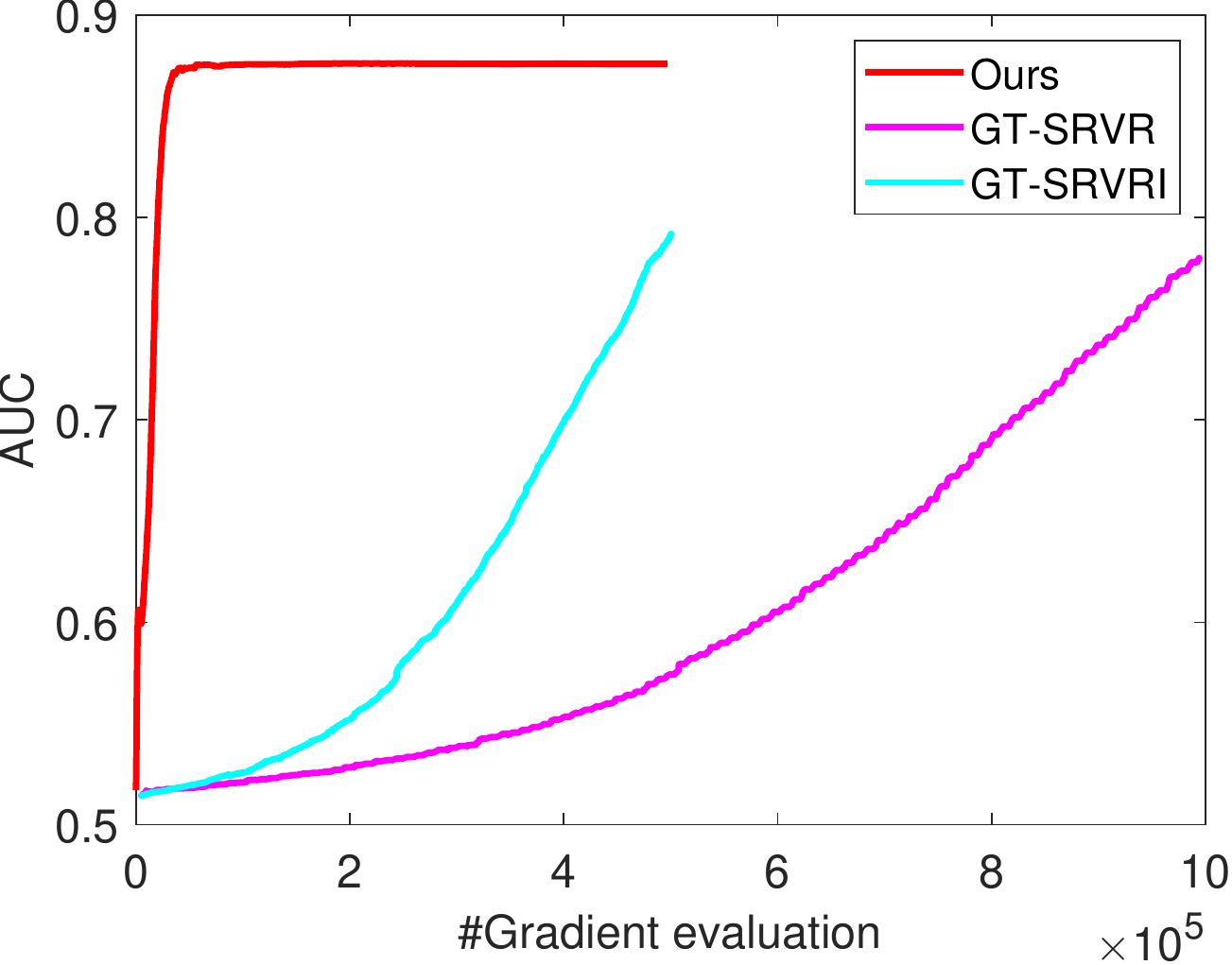}
	}
	\hspace{-10pt}
	\subfigure[ijcnn1]{
		\includegraphics[scale=0.4]{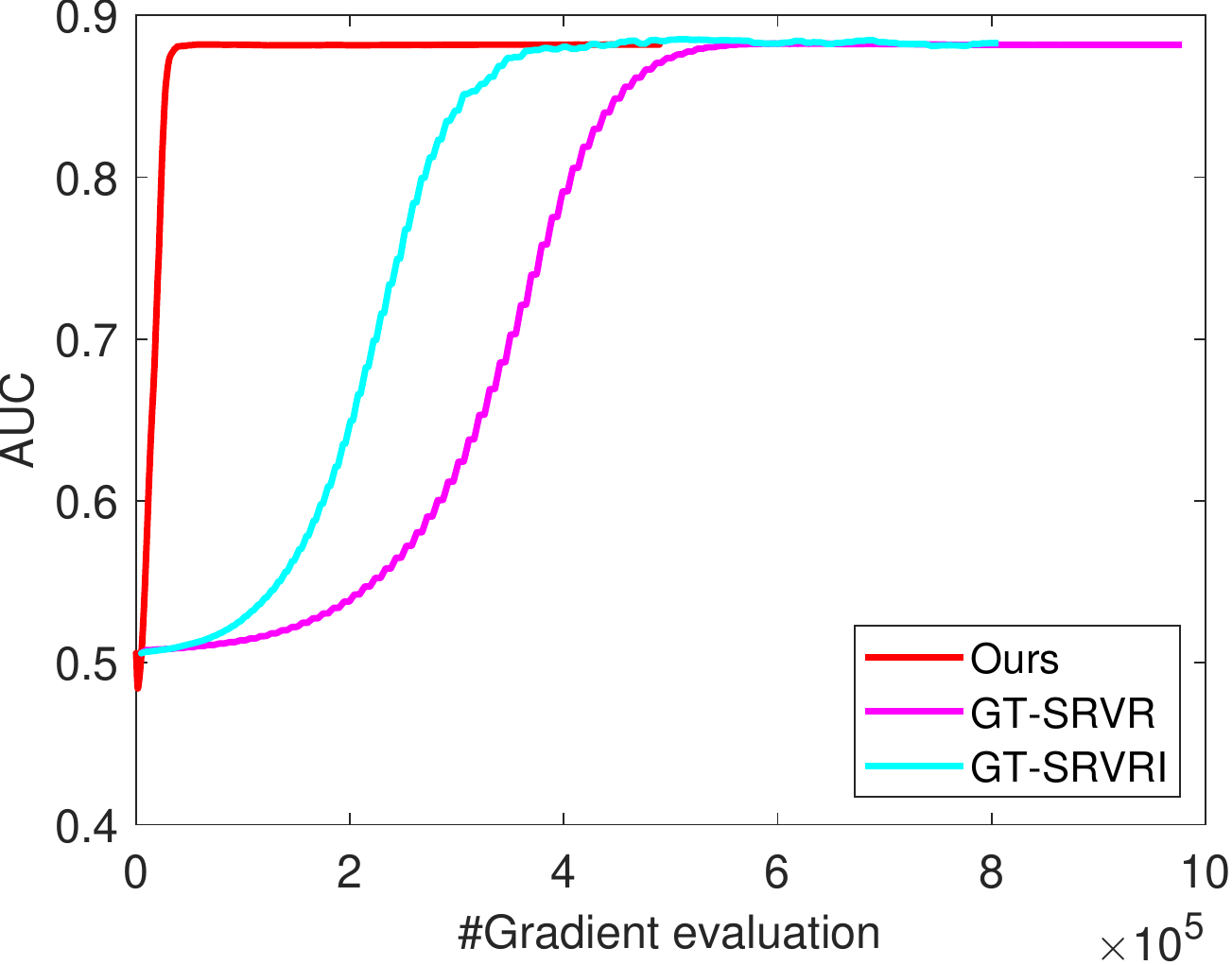}
	}
	\caption{The test AUC versus the number of gradient evaluation when using the random communication graph.}
	\label{fig_rand_grad}
\end{figure*}

\begin{figure}[h]
	\centering 
	\subfigure[Iterations]{
		\includegraphics[scale=0.475]{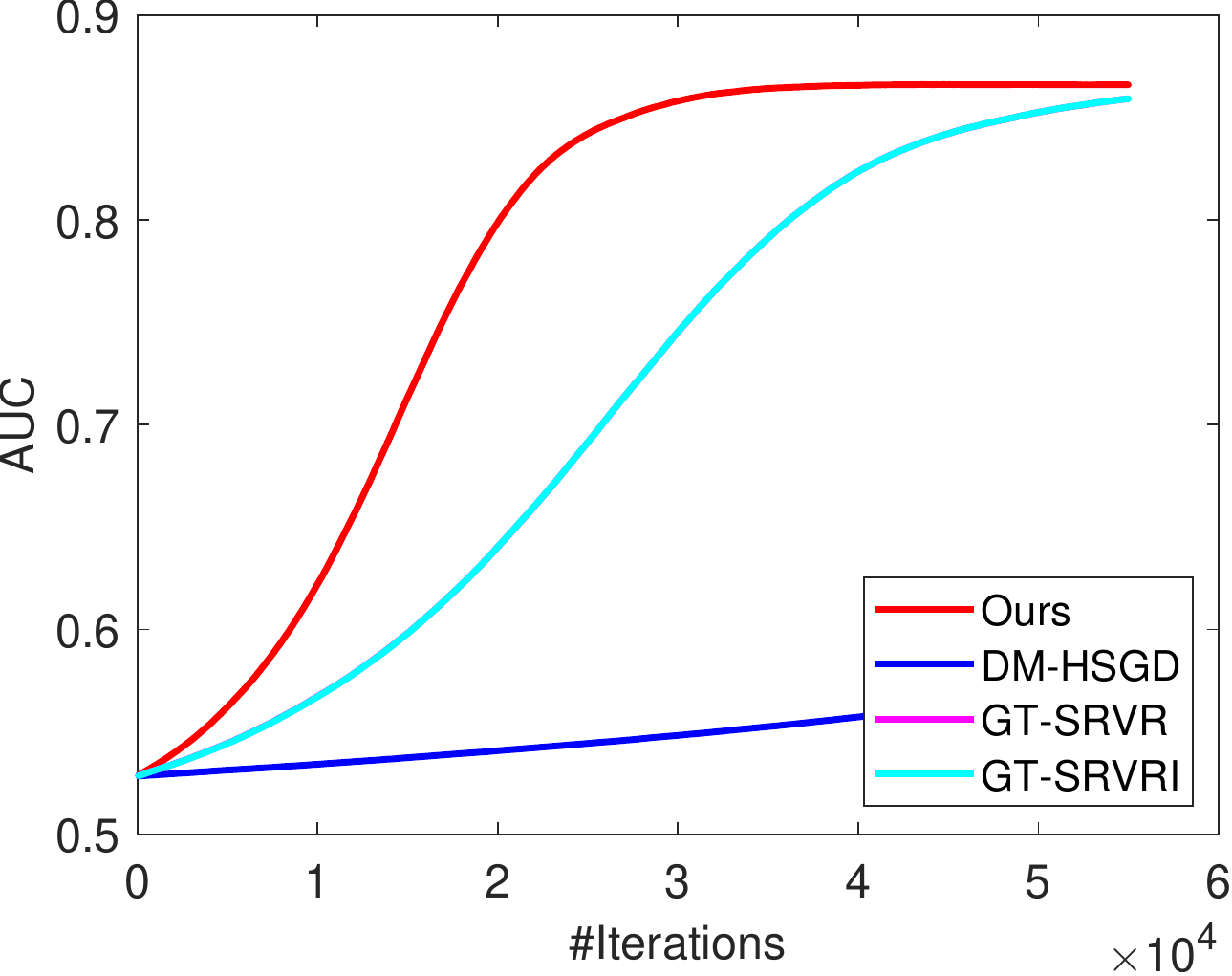}
	}
	\subfigure[Gradient evaluation]{
		\includegraphics[scale=0.475]{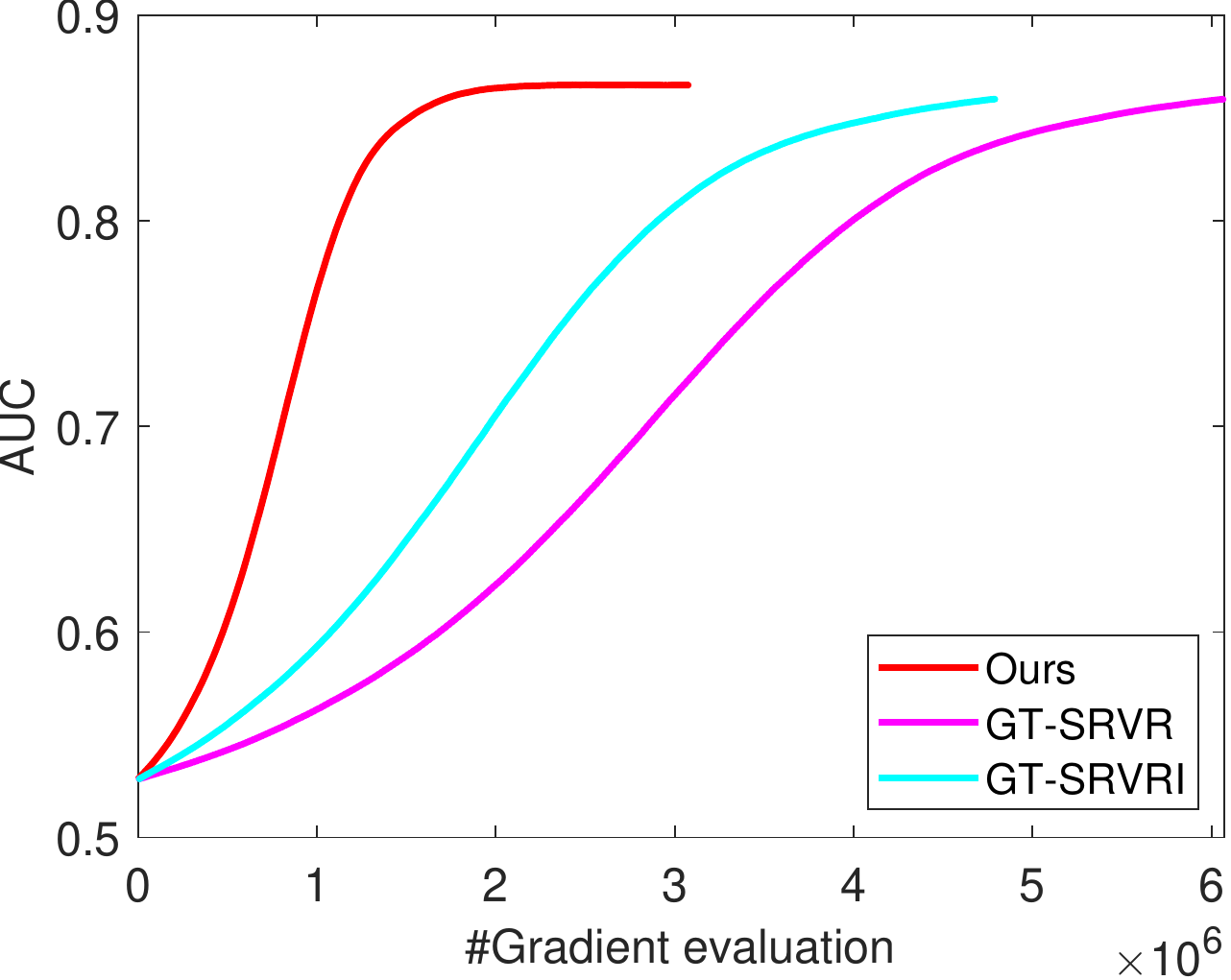}
	}
	\caption{The test AUC  when using the line communication graph for a9a dataset. }
	\label{fig_line_a9a}
\end{figure}

\begin{figure}[h!]
	\centering 
	\subfigure[Iterations]{
		\includegraphics[scale=0.475]{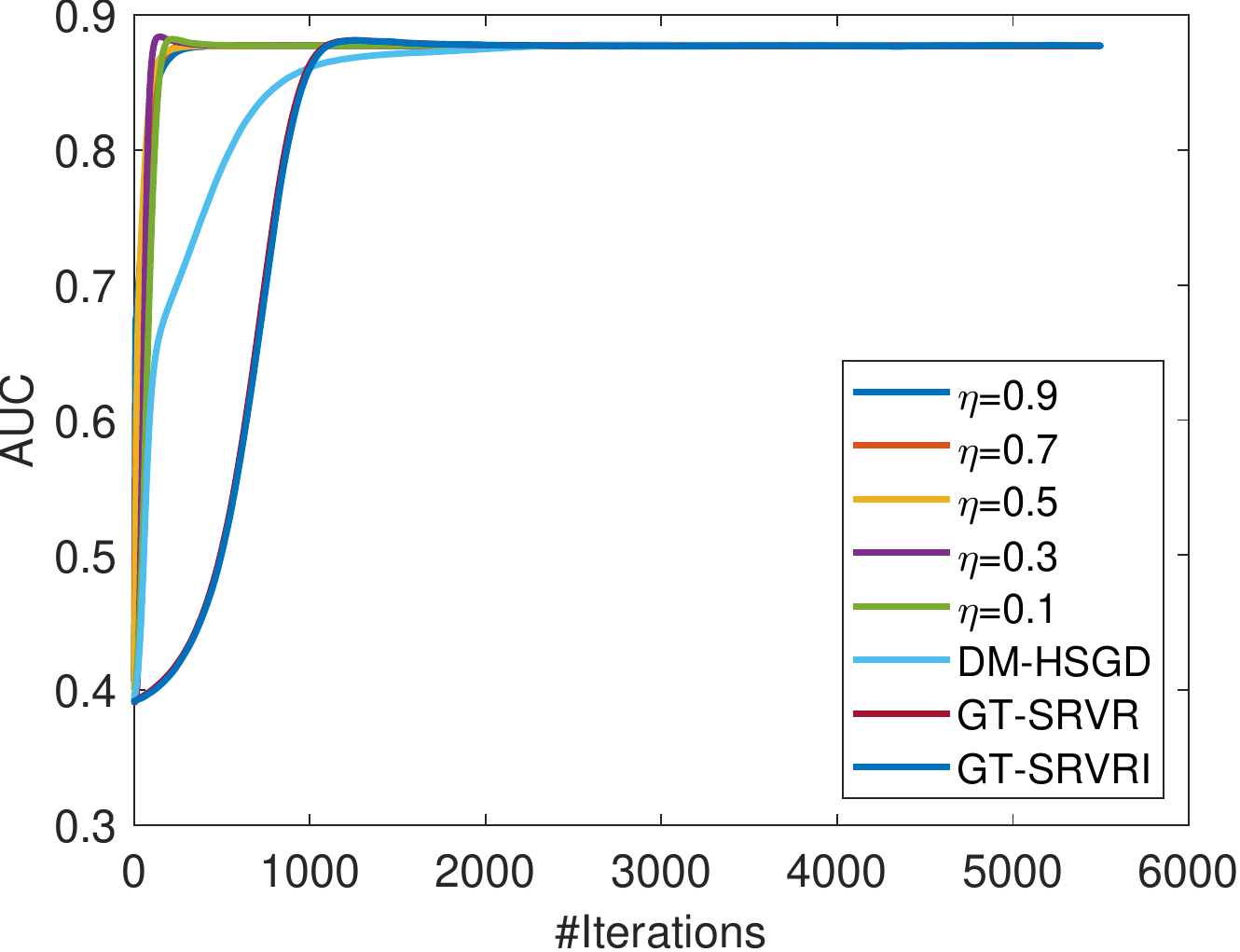}
	}
	\subfigure[Gradient evaluation]{
		\includegraphics[scale=0.475]{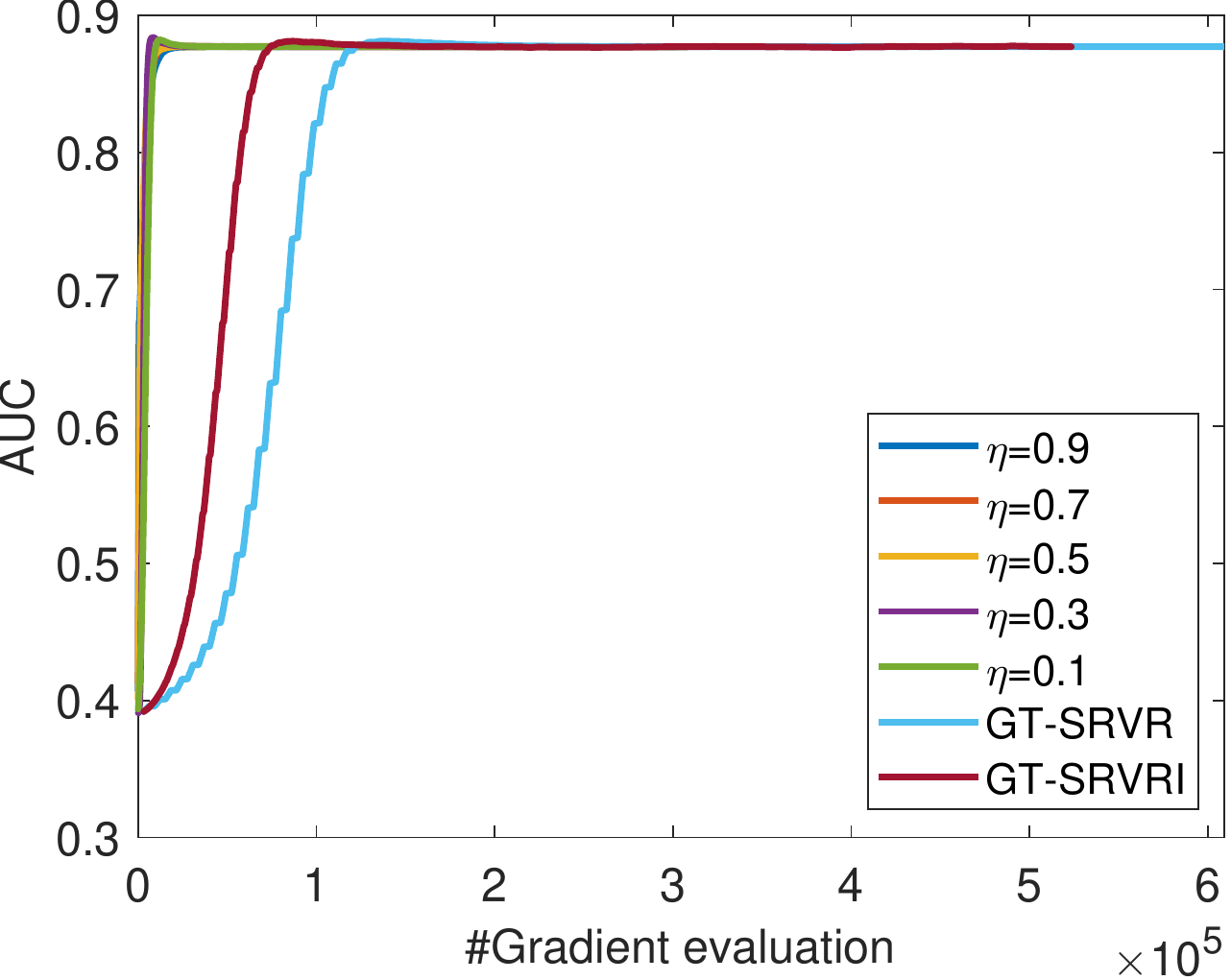}
	}
	
	\caption{The test AUC on different $\eta$ when using the random communication graph for a9a dataset. }
	\label{fig_random_a9a_eta}
\end{figure}

\subsection{Experimental Results}
To compare our method with  baseline methods, we plot  test AUC  versus the number of iterations in Figure~\ref{fig_rand_iter} and that versus the number of gradient evaluation in Figure~\ref{fig_rand_grad} \footnote{We didn't plot the gradient evaluation of DM-HSGD since it is for the stochastic setting, rather than the finite-sum setting.}.  From Figure~\ref{fig_rand_iter}, it can be observed that our method converges much faster than all baseline methods in terms of the number of iterations, which confirms the efficacy of our method. Additionally,  from Figure~\ref{fig_rand_grad}, it can be observed that our method converges faster than GT-SRVR and GT-SRVRI in terms of the number of gradient evaluation, which confirms  the sample complexity of our method is much smaller than theirs.

Moreover, we verify the performance of our method on the line communication network. In particular, in a line communication network, each worker connects with only two neighboring workers. Its spectral gap is in the order of $O(1/K^2)$ \cite{ying2021exponential} where $K=10$ in our experiments. Accordingly, we set the step size  to $(1-\lambda)^2\epsilon=0.01^2\times0.01$ for DM-HSGD, $(1-\lambda)^2=0.01^2$ for other methods. The other hyperparameters are the same with the random communication network. In Figure~\ref{fig_line_a9a}, we plot the test AUC score versus the number of iterations and the number of gradient evaluation for a9a dataset. It can be observed that our method still converges much faster than two baseline methods, which further confirms the efficacy of our method. 

Finally, we demonstrate the performance of our method with different values of $\eta$ in Figure~\ref{fig_random_a9a_eta}. It can be observed that our method with different $\eta$ still converges much faster than baseline methods in terms of the number of iterations and gradient evaluation.

\section{Conclusion}
In this paper, we developed a novel decentralized stochastic gradient descent ascent method for the finite-sum optimization problems. 
We also established the convergence rate of our method, providing the sample complexity and communication complexity. Importantly,  our method can achieve the better communication or computation complexity than existing decentralized methods. Finally, the extensive experimental results confirm the efficacy of our method.

\bibliographystyle{abbrv}
\bibliography{egbib}

\newpage
\onecolumn
\appendix
\section{Supplementary Materials}

Throughout this paper, we  denote  $C_{t}=[\mathbf{c}_{t}^{(1)}, \cdots,  \mathbf{c}_{t}^{(K)} ]$ and $\bar{C}_{t}=[\bar{\mathbf{c}}_{t}, \cdots,  \bar{\mathbf{c}}_{t} ]$ where $\mathbf{c}_{t}^{(k)}$ denotes any variable on the $k$-th device in the $t$-th iteration. 

\begin{lemma} \label{lemma_fx}
	Given Assumption~\ref{assumption_smooth}-\ref{assumption_strong}, by setting $\eta \leq \frac{1}{2\gamma_1 L_{\Phi}}$, we have
	\begin{align}
		& \Phi(\bar{\mathbf{x}}_{t+1}) \leq \Phi(\bar{\mathbf{x}}_{t}) - \frac{\gamma_1\eta}{2} \|\nabla \Phi(\bar{\mathbf{x}}_{t})\|^2-  \frac{\gamma_1\eta}{4} \|\bar{\mathbf{v}}_t\|^2   + \gamma_1\eta L^2 \|\bar{\mathbf{y}}_t -  \mathbf{y}^*(\bar{\mathbf{x}}_t)\|^2\notag \\
		& \quad    +  \frac{2\gamma_1\eta L^2}{K}\sum_{k=1}^{K} \|\bar{\mathbf{x}}_t-\mathbf{x}^{(k)}_{t}\|^2  +  \frac{2\gamma_1\eta L^2}{K}\sum_{k=1}^{K} \|\bar{\mathbf{y}}_t-\mathbf{y}^{(k)}_{t}\|^2  + \frac{2\gamma_1\eta}{K}\sum_{k=1}^{K}\|\nabla_{\mathbf{x}} f^{(k)}(\mathbf{x}^{(k)}_t, \mathbf{y}^{(k)}_t) -\mathbf{v}^{(k)}_t\|^2 \ . 
	\end{align}
\end{lemma}

\begin{proof}
	\begin{align}
		& \Phi(\bar{\mathbf{x}}_{t+1}) \leq \Phi(\bar{\mathbf{x}}_{t}) + \langle \nabla \Phi(\bar{\mathbf{x}}_{t}), \bar{\mathbf{x}}_{t+1} - \bar{\mathbf{x}}_{t}\rangle  + \frac{L_{\Phi}}{2} \|\bar{\mathbf{x}}_{t+1} - \bar{\mathbf{x}}_{t}\|^2 \notag \\
		& = \Phi(\bar{\mathbf{x}}_{t}) - \gamma_1\eta\langle \nabla \Phi(\bar{\mathbf{x}}_{t}), \bar{\mathbf{v}}_t\rangle  + \frac{\gamma_1^2\eta^2L_{\Phi}}{2} \|\bar{\mathbf{v}}_t\|^2 \notag \\
		& = \Phi(\bar{\mathbf{x}}_{t}) - \frac{\gamma_1\eta}{2} \|\nabla \Phi(\bar{\mathbf{x}}_{t})\|^2+ \Big(\frac{\gamma_1^2\eta^2L_{\Phi}}{2}-  \frac{\gamma_1\eta}{2}\Big) \|\bar{\mathbf{v}}_t\|^2  + \frac{\gamma_1\eta}{2} \|\nabla \Phi(\bar{\mathbf{x}}_{t}) -\bar{\mathbf{v}}_t\|^2 \notag \\
		& \leq \Phi(\bar{\mathbf{x}}_{t}) - \frac{\gamma_1\eta}{2} \|\nabla \Phi(\bar{\mathbf{x}}_{t})\|^2+ \Big(\frac{\gamma_1^2\eta^2L_{\Phi}}{2}-  \frac{\gamma_1\eta}{2}\Big) \|\bar{\mathbf{v}}_t\|^2 \notag \\
		& \quad  + \gamma_1\eta\|\nabla \Phi(\bar{\mathbf{x}}_{t}) -\frac{1}{K}\sum_{k=1}^{K}\nabla_{\mathbf{x}} f^{(k)}(\bar{\mathbf{x}}_{t}, \bar{\mathbf{y}})\|^2  + \gamma_1\eta\|\frac{1}{K}\sum_{k=1}^{K}\nabla_{\mathbf{x}} f^{(k)}(\bar{\mathbf{x}}_{t}, \bar{\mathbf{y}}) -\bar{\mathbf{v}}_t\|^2\notag \\
		& \leq \Phi(\bar{\mathbf{x}}_{t}) - \frac{\gamma_1\eta}{2} \|\nabla \Phi(\bar{\mathbf{x}}_{t})\|^2+ \Big(\frac{\gamma_1^2\eta^2L_{\Phi}}{2}-  \frac{\gamma_1\eta}{2}\Big) \|\bar{\mathbf{v}}_t\|^2 \notag \\
		& \quad  + \gamma_1\eta\|\nabla \Phi(\bar{\mathbf{x}}_{t}) -\frac{1}{K}\sum_{k=1}^{K} \nabla_{\mathbf{x}} f^{(k)}(\bar{\mathbf{x}}_{t}, \bar{\mathbf{y}})\|^2  + 2\gamma_1\eta\|\frac{1}{K}\sum_{k=1}^{K}\nabla_{\mathbf{x}} f^{(k)}(\bar{\mathbf{x}}_{t}, \bar{\mathbf{y}}) - \frac{1}{K}\sum_{k=1}^{K}\nabla_{\mathbf{x}} f^{(k)}(\mathbf{x}^{(k)}_t, \mathbf{y}^{(k)}_t)\|^2\notag \\
		& \quad + 2\gamma_1\eta\|\frac{1}{K}\sum_{k=1}^{K}\nabla_{\mathbf{x}} f^{(k)}(\mathbf{x}^{(k)}_t, \mathbf{y}^{(k)}_t) -\bar{\mathbf{v}}_t\|^2\notag \\
		& \leq \Phi(\bar{\mathbf{x}}_{t}) - \frac{\gamma_1\eta}{2} \|\nabla \Phi(\bar{\mathbf{x}}_{t})\|^2-  \frac{\gamma_1\eta}{4} \|\bar{\mathbf{v}}_t\|^2   + \gamma_1\eta L^2 \|\bar{\mathbf{y}}_t -  \mathbf{y}^*(\bar{\mathbf{x}}_t)\|^2 \notag \\
		& \quad  +  \frac{2\gamma_1\eta L^2}{K}\sum_{k=1}^{K} \|\bar{\mathbf{x}}_t-\mathbf{x}^{(k)}_{t}\|^2  +  \frac{2\gamma_1\eta L^2}{K}\sum_{k=1}^{K} \|\bar{\mathbf{y}}_t-\mathbf{y}^{(k)}_{t}\|^2  + \frac{2\gamma_1\eta}{K}\sum_{k=1}^{K}\|\nabla_{\mathbf{x}} f^{(k)}(\mathbf{x}^{(k)}_t, \mathbf{y}^{(k)}_t) -\mathbf{v}^{(k)}_t\|^2 \ , 
	\end{align}
	where the last inequality holds due to $\eta \leq \frac{1}{2\gamma_1 L_{\Phi}}$  and the following inequality:
	\begin{align}
		& \quad 	\|\nabla \Phi(\bar{\mathbf{x}}_{t}) -\frac{1}{K}\sum_{k=1}^{K} \nabla_{\mathbf{x}} f^{(k)}(\bar{\mathbf{x}}_{t}, \bar{\mathbf{y}})\|^2 \notag \\
		& = \|\frac{1}{K}\sum_{k=1}^{K}\nabla_{\mathbf{x}} f^{(k)}(\bar{\mathbf{x}}_{t}, \mathbf{y}^*(\bar{\mathbf{x}}_t)) -\frac{1}{K}\sum_{k=1}^{K}\nabla_{\mathbf{x}} f^{(k)}(\bar{\mathbf{x}}_{t}, \bar{\mathbf{y}})\|^2  \notag \\
		& \leq \frac{L^2}{K}\sum_{k=1}^{K} (\|\bar{\mathbf{x}}_t-\bar{\mathbf{x}}_t\|^2 + \| \mathbf{y}^{*}(\bar{\mathbf{x}}_t)- \bar{\mathbf{y}}\|^2) \notag \\
		& = L^2\| \mathbf{y}^{*}(\bar{\mathbf{x}}_t)- \bar{\mathbf{y}}\|^2 \ . 
	\end{align}
	
\end{proof}

\begin{lemma} \label{lemma_y}
	Given Assumption~\ref{assumption_smooth}-\ref{assumption_strong},  by setting $\gamma_2\leq \frac{1}{6L}$ and $\eta<1$, we have
	\begin{align}
		&  \|\bar{\mathbf{y}}_{t+1} - {\mathbf{y}^*(\overline{\mathbf{x}}_{t+1})}\| ^2   \leq (1-\frac{\eta\gamma_2\mu}{4})  \|\bar{\mathbf{y}}_{t}   - {\mathbf{y}^*(\overline{\mathbf{x}}_t)}\| ^2 - \frac{3\gamma_2^2\eta^2}{4} \|\bar{\mathbf{u}}_t\|^2  + \frac{25\eta\gamma_1^2 \kappa^2 }{6\gamma_2\mu}\| \bar{\mathbf{v}}_{t}  \| ^2 \notag \\
		& \quad  +  \frac{25\eta \gamma_2L^2}{3\mu} \frac{1}{K}\sum_{k=1}^{K}\|\bar{\mathbf{x}}_t - \mathbf{x}^{(k)}_t\|^2  +  \frac{25\eta \gamma_2 L^2}{3\mu} \frac{1}{K}\sum_{k=1}^{K}\|\bar{\mathbf{y}}_t -  {\mathbf{y}}^{(k)}_t \|^2  \notag \\
		& \quad + \frac{25\eta \gamma_2 }{3\mu} \frac{1}{K}\sum_{k=1}^{K}\|\nabla_{\mathbf{y}} f^{(k)}({\mathbf{x}}^{(k)}_t, {\mathbf{y}}^{(k)}_t)  -{\mathbf{u}}^{(k)}_t\|^2   \ .  
	\end{align}
\end{lemma}

\begin{proof}
	By setting $\gamma_2\leq \frac{1}{6L}$ and $\eta<1$, we have
	\begin{align}
		&  \quad \|\bar{\mathbf{y}}_{t+1} - {\mathbf{y}^*(\overline{\mathbf{x}}_{t+1})}\| ^2 \notag \\
		& \leq  (1-\frac{\eta\gamma_2\mu}{4})  \|\bar{\mathbf{y}}_{t}   - {\mathbf{y}^*(\overline{\mathbf{x}}_t)}\| ^2 - \frac{3\gamma_2^2\eta}{4} \|\bar{\mathbf{u}}_t\|^2   + \frac{25\eta\gamma_1^2 \kappa^2 }{6\gamma_2\mu}\| \bar{\mathbf{v}}_{t}  \| ^2  +  \frac{25\eta \gamma_2}{6\mu} \|\nabla_{\mathbf{y}} f(\bar{\mathbf{x}}_t, \bar{\mathbf{y}}_t)-\bar{\mathbf{u}}_t\|^2   \notag \\
		& \leq (1-\frac{\eta\gamma_2\mu}{4})  \|\bar{\mathbf{y}}_{t}   - {\mathbf{y}^*(\overline{\mathbf{x}}_t)}\| ^2 - \frac{3\gamma_2^2\eta}{4} \|\bar{\mathbf{u}}_t\|^2   + \frac{25\eta\gamma_1^2 \kappa^2 }{6\gamma_2\mu}\| \bar{\mathbf{v}}_{t}  \| ^2 \notag \\
		& \quad  +  \frac{25\eta \gamma_2}{3\mu} \|\nabla_{\mathbf{y}} f(\bar{\mathbf{x}}_t, \bar{\mathbf{y}}_t) - \frac{1}{K}\sum_{k=1}^{K}\nabla_{\mathbf{y}} f^{(k)}({\mathbf{x}}^{(k)}_t, {\mathbf{y}}^{(k)}_t) \|^2\notag \\
		& \quad + \frac{25\eta \gamma_2}{3\mu} \frac{1}{K}\sum_{k=1}^{K}\|\nabla_{\mathbf{y}} f^{(k)}({\mathbf{x}}^{(k)}_t, {\mathbf{y}}^{(k)}_t)  -{\mathbf{u}}^{(k)}_t\|^2   \notag \\
		& \leq (1-\frac{\eta\gamma_2\mu}{4})  \|\bar{\mathbf{y}}_{t}   - {\mathbf{y}^*(\overline{\mathbf{x}}_t)}\| ^2 - \frac{3\gamma_2^2\eta}{4} \|\bar{\mathbf{u}}_t\|^2   + \frac{25\eta\gamma_1^2 \kappa^2 }{6\gamma_2\mu}\| \bar{\mathbf{v}}_{t}  \| ^2 \notag \\
		& \quad  +  \frac{25\eta \gamma_2L^2}{3\mu} \frac{1}{K}\sum_{k=1}^{K}\|\bar{\mathbf{x}}_t - \mathbf{x}^{(k)}_t\|^2+  \frac{25\eta \gamma_2 L^2}{3\mu} \frac{1}{K}\sum_{k=1}^{K}\|\bar{\mathbf{y}}_t -  {\mathbf{y}}^{(k)}_t \|^2\notag \\
		& \quad + \frac{25\eta \gamma_2 }{3\mu} \frac{1}{K}\sum_{k=1}^{K}\|\nabla_{\mathbf{y}} f^{(k)}({\mathbf{x}}^{(k)}_t, {\mathbf{y}}^{(k)}_t)  -{\mathbf{u}}^{(k)}_t\|^2    \ ,  
	\end{align}
	where the first step follows from  \cite{huang2020gradient}.

\end{proof}

\begin{lemma} \label{lemma_consensus_a}
	Given Assumption~\ref{graph}-\ref{assumption_strong}, for $t>0$, by setting $s_t=s_1$, $\rho_{t}=\rho_1$, we have
	\begin{align}
		&  \sum_{k=1}^{K} \mathbb{E}[\|\mathbf{a}_{t+1}^{(k)} - \bar{\mathbf{a}}_{t+1}\|^2] \leq  \frac{1+\lambda^2}{2}\sum_{k=1}^{K} \mathbb{E}[\|\mathbf{a}_{t}^{(k)} - \bar{\mathbf{a}}_{t}\|^2]\notag \\
		& \quad  +  \frac{6L^2}{1-\lambda^2} \sum_{k=1}^{K}  \mathbb{E}[\|\mathbf{x}_{t+1}^{(k)} - \mathbf{x}_{t}^{(k)}\|^2] + \frac{6L^2}{1-\lambda^2} \sum_{k=1}^{K} \mathbb{E}[\|\mathbf{y}_{t+1}^{(k)} - \mathbf{y}_{t}^{(k)}\|^2] \notag \\
		& \quad + \frac{6\rho_{1}^2 }{1-\lambda^2} \sum_{k=1}^{K} \mathbb{E}[\| \mathbf{v}_{t}^{(k)} - \nabla_{\mathbf{x}} f^{(k)}(\mathbf{x}_{t}^{(k)}, \mathbf{y}_{t}^{(k)})\|^2 ] +\frac{6\rho_{1}^2}{(1-\lambda^2)s_1} \sum_{k=1}^{K} \frac{1}{n}\sum_{j=1}^{n} \mathbb{E}[\|\nabla_{\mathbf{x}} f^{(k)}_{j}(\mathbf{x}_{t}^{(k)}, \mathbf{y}_{t}^{(k)})-\mathbf{g}_{j, t}^{(k)}\|^2] \ . 
	\end{align}
	
\end{lemma}

\begin{proof}
	Based on the gradient tracking scheme in Algorithm~\ref{dsgda},we have
	\begin{align}
		& \quad \sum_{k=1}^{K} \mathbb{E}[\|\mathbf{a}_{t+1}^{(k)} - \bar{\mathbf{a}}_{t+1}\|^2] = \mathbb{E}[\|A_{t+1}-\bar{A}_{t+1}\|_F^2]\notag \\
		& = \mathbb{E}[\|A_{t}W+ V_{t+1}- V_{t} - (\bar{A}_{t}+\bar{V}_{t+1}-\bar{V}_t)\|_F^2] \notag \\
		& \leq (1+a)\mathbb{E}[\|A_{t}W - \bar{A}_{t}\|_F^2 ]+ (1+\frac{1}{a})\mathbb{E}[\|V_{t+1}- V_{t}  - \bar{V}_{t+1}+\bar{V}_t\|_F^2] \notag \\
		& \leq (1+a)\lambda^2\mathbb{E}[\|A_{t} - \bar{A}_{t}\|_F^2] + (1+\frac{1}{a})\mathbb{E}[\|V_{t+1}- V_{t} \|_F^2] \notag \\
		& = \frac{1+\lambda^2}{2}\mathbb{E}[\|A_{t} - \bar{A}_{t}\|_F^2 ]+ \frac{2}{1-\lambda^2} \mathbb{E}[\|V_{t+1}- V_{t} \|_F^2] \notag \\
		& = \frac{1+\lambda^2}{2}\sum_{k=1}^{K} \mathbb{E}[\|\mathbf{a}_{t}^{(k)} - \bar{\mathbf{a}}_{t}\|^2] +  \frac{2}{1-\lambda^2} \sum_{k=1}^{K} \mathbb{E}[\|\mathbf{v}_{t+1}^{(k)} - \mathbf{v}_{t}^{(k)} \|^2] \ , 
	\end{align}
	where the second to last step holds due to $a=\frac{1-\lambda^2}{2\lambda^2}$.  In the following, we bound $\mathbb{E}[\|\mathbf{v}_{t+1}^{(k)} - \mathbf{v}_{t}^{(k)} \|^2]$.
	\begin{align}
		&\quad \mathbb{E}[\|\mathbf{v}_{t+1}^{(k)} - \mathbf{v}_{t}^{(k)} \|^2] \notag \\
		& = \mathbb{E}[\|\frac{1}{s_{t+1}} \sum_{i \in \mathcal{S}_{t+1}}\Big(\nabla_{\mathbf{x}} f^{(k)}_{i}(\mathbf{x}_{t+1}^{(k)}, \mathbf{y}_{t+1}^{(k)})-\nabla_{\mathbf{x}} f^{(k)}_{i}(\mathbf{x}_{t}^{(k)}, \mathbf{y}_{t}^{(k)})\Big)+(1-\rho_{t+1}) \mathbf{v}_{t}^{(k)}\notag \\
		& \quad +\rho_{t+1}\Big(\frac{1}{s_{t+1}} \sum_{i \in \mathcal{S}_{t+1}}(\nabla_{\mathbf{x}} f^{(k)}_{i}(\mathbf{x}_{t}^{(k)}, \mathbf{y}_{t}^{(k)})-\mathbf{g}_{i, t}^{(k)})+\frac{1}{n} \sum_{j=1}^{n} \mathbf{g}_{j, t}^{(k)}\Big) - \mathbf{v}_{t}^{(k)} \|^2 ]\notag \\
		& = \mathbb{E}[\|\frac{1}{s_{t+1}} \sum_{i \in \mathcal{S}_{t+1}}\Big(\nabla_{\mathbf{x}} f^{(k)}_{i}(\mathbf{x}_{t+1}^{(k)}, \mathbf{y}_{t+1}^{(k)})-\nabla_{\mathbf{x}} f^{(k)}_{i}(\mathbf{x}_{t}^{(k)}, \mathbf{y}_{t}^{(k)})\Big)-\rho_{t+1} \mathbf{v}_{t}^{(k)} + \rho_{t+1} \nabla_{\mathbf{x}} f^{(k)}(\mathbf{x}_{t}^{(k)}, \mathbf{y}_{t}^{(k)})\notag \\
		& \quad +\rho_{t+1}\Big(\frac{1}{s_{t+1}} \sum_{i \in \mathcal{S}_{t+1}}(\nabla_{\mathbf{x}} f^{(k)}_{i}(\mathbf{x}_{t}^{(k)}, \mathbf{y}_{t}^{(k)})-\mathbf{g}_{i, t}^{(k)})+\frac{1}{n} \sum_{j=1}^{n} \mathbf{g}_{j, t}^{(k)}- \nabla_{\mathbf{x}} f^{(k)}(\mathbf{x}_{t}^{(k)}, \mathbf{y}_{t}^{(k)})\Big)  \|^2 ]\notag \\
		& \leq 3 \mathbb{E}[\|\frac{1}{s_{t+1}} \sum_{i \in \mathcal{S}_{t+1}}\Big(\nabla_{\mathbf{x}} f^{(k)}_{i}(\mathbf{x}_{t+1}^{(k)}, \mathbf{y}_{t+1}^{(k)})-\nabla_{\mathbf{x}} f^{(k)}_{i}(\mathbf{x}_{t}^{(k)}, \mathbf{y}_{t}^{(k)})\Big)\|^2 ]\notag \\
		& \quad + 3\mathbb{E}[\|-\rho_{t+1} \mathbf{v}_{t}^{(k)} + \rho_{t+1} \nabla_{\mathbf{x}} f^{(k)}(\mathbf{x}_{t}^{(k)}, \mathbf{y}_{t}^{(k)})\|^2] \notag \\
		& \quad + 3\mathbb{E}[\|\rho_{t+1}\Big(\frac{1}{s_{t+1}} \sum_{i \in \mathcal{S}_{t+1}}(\nabla_{\mathbf{x}} f^{(k)}_{i}(\mathbf{x}_{t}^{(k)}, \mathbf{y}_{t}^{(k)})-\mathbf{g}_{i, t}^{(k)})+\frac{1}{n} \sum_{j=1}^{n} \mathbf{g}_{j, t}^{(k)}- \nabla_{\mathbf{x}} f^{(k)}(\mathbf{x}_{t}^{(k)}, \mathbf{y}_{t}^{(k)})\Big)  \|^2]\notag \\
		& \leq 3L^2 \mathbb{E}[\|\mathbf{x}_{t+1}^{(k)} - \mathbf{x}_{t}^{(k)}\|^2 ]+3L^2 \mathbb{E}[\|\mathbf{y}_{t+1}^{(k)} - \mathbf{y}_{t}^{(k)}\|^2] + 3\rho_{t+1}^2 \mathbb{E}[\| \mathbf{v}_{t}^{(k)} - \nabla_{\mathbf{x}} f^{(k)}(\mathbf{x}_{t}^{(k)}, \mathbf{y}_{t}^{(k)})\|^2] \notag \\
		& \quad + 3\rho_{t+1}^2\mathbb{E}[\|\frac{1}{s_{t+1}} \sum_{i \in \mathcal{S}_{t+1}}(\nabla_{\mathbf{x}} f^{(k)}_{i}(\mathbf{x}_{t}^{(k)}, \mathbf{y}_{t}^{(k)})-\mathbf{g}_{i, t}^{(k)})+\frac{1}{n} \sum_{j=1}^{n} \mathbf{g}_{j, t}^{(k)}- \nabla_{\mathbf{x}} f^{(k)}(\mathbf{x}_{t}^{(k)}, \mathbf{y}_{t}^{(k)})  \|^2] \  . 
	\end{align}
	Then, we bound $\mathbb{E}[\|\frac{1}{s_{t+1}} \sum_{i \in \mathcal{S}_{t+1}}(\nabla_{\mathbf{x}} f^{(k)}_{i}(\mathbf{x}_{t}^{(k)}, \mathbf{y}_{t}^{(k)})-\mathbf{g}_{i, t}^{(k)})+\frac{1}{n} \sum_{j=1}^{n} \mathbf{g}_{j, t}^{(k)}- \nabla_{\mathbf{x}} f^{(k)}(\mathbf{x}_{t}^{(k)}, \mathbf{y}_{t}^{(k)}) \|^2]$ as follows:
	\begin{align}
		& \quad \mathbb{E}[\|\frac{1}{s_{t+1}} \sum_{i \in \mathcal{S}_{t+1}}(\nabla_{\mathbf{x}} f^{(k)}_{i}(\mathbf{x}_{t}^{(k)}, \mathbf{y}_{t}^{(k)})-\mathbf{g}_{i, t}^{(k)})+\frac{1}{n} \sum_{j=1}^{n} \mathbf{g}_{j, t}^{(k)}- \nabla_{\mathbf{x}} f^{(k)}(\mathbf{x}_{t}^{(k)}, \mathbf{y}_{t}^{(k)}) \|^2] \notag \\
		& = \mathbb{E}[\|\frac{1}{s_{t+1}} \sum_{i \in \mathcal{S}_{t+1}}\Big(\nabla_{\mathbf{x}} f^{(k)}_{i}(\mathbf{x}_{t}^{(k)}, \mathbf{y}_{t}^{(k)})-\mathbf{g}_{i, t}^{(k)}+\frac{1}{n} \sum_{j=1}^{n} \mathbf{g}_{j, t}^{(k)}- \frac{1}{n} \sum_{j=1}^{n} \nabla_{\mathbf{x}} f_j^{(k)}(\mathbf{x}_{t}^{(k)}, \mathbf{y}_{t}^{(k)}) \Big)\|^2]  \notag \\
		& = \frac{1}{s^2_{t+1}}\sum_{i \in \mathcal{S}_{t+1}}\mathbb{E}[\| \nabla_{\mathbf{x}} f^{(k)}_{i}(\mathbf{x}_{t}^{(k)}, \mathbf{y}_{t}^{(k)})-\mathbf{g}_{i, t}^{(k)}+\frac{1}{n} \sum_{j=1}^{n} \mathbf{g}_{j, t}^{(k)}- \frac{1}{n} \sum_{j=1}^{n} \nabla_{\mathbf{x}} f_j^{(k)}(\mathbf{x}_{t}^{(k)}, \mathbf{y}_{t}^{(k)}) \|^2]  \notag \\
		& \leq  \frac{1}{s^2_{t+1}}\sum_{i \in \mathcal{S}_{t+1}}\frac{1}{n} \sum_{j=1}^{n} \mathbb{E}[\|\mathbf{g}_{j, t}^{(k)}-\nabla_{\mathbf{x}} f_j^{(k)}(\mathbf{x}_{t}^{(k)}, \mathbf{y}_{t}^{(k)}) \|^2]   \notag \\
		& =  \frac{1}{s_{t+1}}\frac{1}{n} \sum_{j=1}^{n} \mathbb{E}[\|\mathbf{g}_{j, t}^{(k)}-\nabla_{\mathbf{x}} f_j^{(k)}(\mathbf{x}_{t}^{(k)}, \mathbf{y}_{t}^{(k)}) \|^2]  \ , 
	\end{align}
	where the second step holds because $\mathbf{z}_i=\nabla_{\mathbf{x}} f^{(k)}_{i}(\mathbf{x}_{t}^{(k)}, \mathbf{y}_{t}^{(k)})-\mathbf{g}_{i, t}^{(k)}+\frac{1}{n} \sum_{j=1}^{n} \mathbf{g}_{j, t}^{(k)}- \frac{1}{n} \sum_{j=1}^{n} \nabla_{\mathbf{x}} f_j^{(k)}(\mathbf{x}_{t}^{(k)}, \mathbf{y}_{t}^{(k)})$ are independent random vectors with zero mean values, the third step holds due to $\mathbb{E}[\|\mathbf{z}-\mathbb{E}[\mathbf{z}]\|^2]\leq \mathbb{E}[\|\mathbf{z}\|^2]$ for $\mathbf{z}=\nabla_{\mathbf{x}} f^{(k)}_{i}(\mathbf{x}_{t}^{(k)}, \mathbf{y}_{t}^{(k)})-\mathbf{g}_{i, t}^{(k)}$. 
	
	By combining above two inequalities, we have
	\begin{align}
		&\quad \mathbb{E}[\|\mathbf{v}_{t+1}^{(k)} - \mathbf{v}_{t}^{(k)} \|^2] \notag \\
		& \leq 3L^2 \mathbb{E}[\|\mathbf{x}_{t+1}^{(k)} - \mathbf{x}_{t}^{(k)}\|^2 ]+3L^2 \mathbb{E}[\|\mathbf{y}_{t+1}^{(k)} - \mathbf{y}_{t}^{(k)}\|^2] + 3\rho_{t+1}^2 \mathbb{E}[\| \mathbf{v}_{t}^{(k)} - \nabla_{\mathbf{x}} f^{(k)}(\mathbf{x}_{t}^{(k)}, \mathbf{y}_{t}^{(k)})\|^2] \notag \\
		& \quad  +\frac{ 3\rho_{t+1}^2}{s_{t+1}}\frac{1}{n}\sum_{j=1}^{n}\mathbb{E}[ \|\nabla_{\mathbf{x}} f^{(k)}_{j}(\mathbf{x}_{t}^{(k)}, \mathbf{y}_{t}^{(k)})-\mathbf{g}_{j, t}^{(k)}\|^2] \  .
	\end{align}

	Then, by setting $s_t=s_1$, $\rho_{t}=\rho_1$, we have
	\begin{align}
		&  \sum_{k=1}^{K} \mathbb{E}[\|\mathbf{a}_{t+1}^{(k)} - \bar{\mathbf{a}}_{t+1}\|^2]\leq  \frac{1+\lambda^2}{2}\sum_{k=1}^{K}  \mathbb{E}[\|\mathbf{a}_{t}^{(k)} - \bar{\mathbf{a}}_{t}\|^2]  \notag \\
		& \quad +  \frac{6L^2}{1-\lambda^2} \sum_{k=1}^{K}   \mathbb{E}[\|\mathbf{x}_{t+1}^{(k)} - \mathbf{x}_{t}^{(k)}\|^2] + \frac{6L^2}{1-\lambda^2} \sum_{k=1}^{K}  \mathbb{E}[\|\mathbf{y}_{t+1}^{(k)} - \mathbf{y}_{t}^{(k)}\|^2 ] \notag \\
		& \quad + \frac{6\rho_{1}^2 }{1-\lambda^2} \sum_{k=1}^{K}  \mathbb{E}[\| \mathbf{v}_{t}^{(k)} - \nabla_{\mathbf{x}} f^{(k)}(\mathbf{x}_{t}^{(k)}, \mathbf{y}_{t}^{(k)})\|^2] +\frac{6\rho_{1}^2}{(1-\lambda^2)s_1} \sum_{k=1}^{K} \frac{1}{n}\sum_{j=1}^{n}  \mathbb{E}[\|\nabla_{\mathbf{x}} f^{(k)}_{j}(\mathbf{x}_{t}^{(k)}, \mathbf{y}_{t}^{(k)})-\mathbf{g}_{j, t}^{(k)}\|^2] \ . 
	\end{align}

\end{proof}

\begin{lemma} \label{lemma_consensus_b}
	Given Assumption~\ref{graph}-\ref{assumption_strong}, for $t>0$, by setting $s_t=s_1$, $\rho_{t}=\rho_1$, we have
	\begin{align}
		&  \sum_{k=1}^{K} \mathbb{E}[\|\mathbf{b}_{t+1}^{(k)} - \bar{\mathbf{b}}_{t+1}\|^2]\leq  \frac{1+\lambda^2}{2}\sum_{k=1}^{K}\mathbb{E}[ \|\mathbf{b}_{t}^{(k)} - \bar{\mathbf{b}}_{t}\|^2] \notag \\
		& \quad +  \frac{6L^2}{1-\lambda^2} \sum_{k=1}^{K}  \mathbb{E}[\|\mathbf{x}_{t+1}^{(k)} - \mathbf{x}_{t}^{(k)}\|^2] + \frac{6L^2}{1-\lambda^2} \sum_{k=1}^{K} \mathbb{E}[\|\mathbf{y}_{t+1}^{(k)} - \mathbf{y}_{t}^{(k)}\|^2] \notag \\
		& \quad + \frac{6\rho_{1}^2 }{1-\lambda^2} \sum_{k=1}^{K} \mathbb{E}[\| \mathbf{u}_{t}^{(k)} - \nabla_{\mathbf{y}} f^{(k)}(\mathbf{x}_{t}^{(k)}, \mathbf{y}_{t}^{(k)})\|^2]   +\frac{6\rho_{1}^2}{(1-\lambda^2)s_1} \sum_{k=1}^{K} \frac{1}{n}\sum_{j=1}^{n} \mathbb{E}[\|\nabla_{\mathbf{y}} f^{(k)}_{j}(\mathbf{x}_{t}^{(k)}, \mathbf{y}_{t}^{(k)})-\mathbf{h}_{j, t}^{(k)}\|^2] \ .
	\end{align}
	
\end{lemma}

Similarly, we can prove the inequality with respect to $ \mathbb{E}[\|\mathbf{b}_{t+1}^{(k)} - \bar{\mathbf{b}}_{t+1}\|^2]$ by following Lemma~\ref{lemma_consensus_a}.  

\begin{lemma} \label{lemma_var_x}
	Given Assumption~\ref{graph}-\ref{assumption_strong}, for $t>0$, by setting $s_{t}=s_1$,  and $\rho_{t}=\rho_1$, we have
	\begin{align}
		&   \mathbb{E}[\|\mathbf{v}_{t+1}^{(k)}-\nabla_{\mathbf{x}} f^{(k)}(\mathbf{x}_{t+1}^{(k)}, \mathbf{y}_{t+1}^{(k)})\|^2 ] \leq (1-\rho_{1})^2 \mathbb{E}[\| \mathbf{v}_{t}^{(k)}- \nabla_{\mathbf{x}} f^{(k)}(\mathbf{x}_{t}^{(k)}, \mathbf{y}_{t}^{(k)})\|^2] \notag \\
		& \quad + \frac{2L^2}{s_{1}} \mathbb{E}[\|\mathbf{x}_{t+1}^{(k)}- \mathbf{x}_{t}^{(k)}\|^2 ]+ \frac{2L^2}{s_{1}} \mathbb{E}[\|\mathbf{y}_{t+1}^{(k)}-\mathbf{y}_{t}^{(k)}\|^2 ] +\frac{ 2\rho_{1}^2}{s_{1}}\frac{1}{n}\sum_{i=1}^{n} \mathbb{E}[ \|\nabla_{\mathbf{x}} f^{(k)}_{i}(\mathbf{x}_{t}^{(k)}, \mathbf{y}_{t}^{(k)})-\mathbf{g}_{i, t}^{(k)}\|^2 ] \ . 
	\end{align}
	When $t=0$, $\rho_{0}=1$, we have
	\begin{align}
		&  \mathbb{E}[\|\mathbf{v}^{(k)}_{0}-\nabla_{\mathbf{x}} f^{(k)}(\mathbf{x}_{0}^{(k)}, \mathbf{y}_{0}^{(k)})\|^2]  \leq  \frac{n-s_0}{(n-1)s_0}\frac{1}{n}\sum_{i=1}^{n} \|\nabla_{\mathbf{x}} f^{(k)}_i(\mathbf{x}_{0}^{(k)}, \mathbf{y}_{0}^{(k)})\|^2 \ .
	\end{align}
\end{lemma}

\begin{proof}
	When $t=0$,  $\rho_{0}=1$,  by denoting $\mathbf{z}_i =f^{(k)}_{i}(\mathbf{x}_{0}^{(k)}, \mathbf{y}_{0}^{(k)}) -\nabla_{\mathbf{x}} f^{(k)}(\mathbf{x}_{0}^{(k)}, \mathbf{y}_{0}^{(k)})$, we can get
	\begin{align}
		& \quad \mathbb{E}[\|\mathbf{v}^{(k)}_{0}-\nabla_{\mathbf{x}} f^{(k)}(\mathbf{x}_{0}^{(k)}, \mathbf{y}_{0}^{(k)})\|^2] \notag \\
		& = \mathbb{E}[\|\frac{1}{s_0} \sum_{i \in \mathcal{S}_{0}}\nabla_{\mathbf{x}} f^{(k)}_{i}(\mathbf{x}_{0}^{(k)}, \mathbf{y}_{0}^{(k)}) -\nabla_{\mathbf{x}} f^{(k)}(\mathbf{x}_{0}^{(k)}, \mathbf{y}_{0}^{(k)})\|^2]  \notag \\
		& =\mathbb{E}[\|  \frac{1}{s_{0}}\sum_{i=1}^{n} \mathbf{z}_i \mathbb{I}_{i\in \mathcal{S}_0} \|^2]  \notag \\
		& = \frac{1}{s_{0}^2}\mathbb{E}[\| \sum_{i=1}^{n} \mathbf{z}_i \mathbb{I}_{i\in \mathcal{S}_0}  \|^2] \notag \\
		& = \frac{1}{s_{0}^2}\mathbb{E}\left[\sum_{i=1}^{n} \left\|  \mathbf{z}_i \mathbb{I}_{i\in \mathcal{S}_0} \right \|^2 +  \sum_{i=1}^{n}\sum_{j=1, j\neq i}^{n}\langle \mathbf{z}_i \mathbb{I}_{i\in \mathcal{S}_0} , \mathbf{z}_j\mathbb{I}_{j\in \mathcal{S}_0} \rangle\right]\notag \\
		& = \frac{1}{s_{0}^2}\mathbb{E}\left[\sum_{i=1}^{n} \left\|  \mathbf{z}_i \right \|^2  \mathbb{I}_{j\in \mathcal{S}_0}  +  \sum_{i=1}^{n}\sum_{j=1, j\neq i}^{n}\langle \mathbf{z}_i , \mathbf{z}_j \rangle \mathbb{I}_{i\in \mathcal{S}_0, j\in \mathcal{S}_0}\right] \notag \\
		& = \frac{1}{s_{0}^2}\frac{s_{0}}{n}\sum_{i=1}^{n} \left\|  \mathbf{z}_i \right \|^2  +  \frac{1}{s_{0}^2} \frac{s_{0}(s_{0}-1)}{n(n-1)}\sum_{i=1}^{n}\sum_{j=1, j\neq i}^{n}\langle \mathbf{z}_i , \mathbf{z}_j \rangle \notag\\
		& = \frac{1}{s_{0}n}\sum_{i=1}^{n} \left\|  \mathbf{z}_i \right \|^2  +   \frac{(s_{0}-1)}{s_{0}n(n-1)}\sum_{i=1}^{n}\sum_{j=1, j\neq i}^{n}\langle \mathbf{z}_i , \mathbf{z}_j \rangle \notag \\
		& = \frac{1}{s_{0}n}\sum_{i=1}^{n} \left\|  \mathbf{z}_i \right \|^2  +   \frac{(s_{0}-1)}{s_{0}n(n-1)}\left( \left\|  \sum_{i=1}^{n}\mathbf{z}_i \right \|^2 - \sum_{i=1}^{n} \left\|  \mathbf{z}_i \right \|^2\right) \notag \\
		& = \left(\frac{1}{s_{0}n} -   \frac{(s_{0}-1)}{s_{0}n(n-1)}\right)\sum_{i=1}^{n} \left\|  \mathbf{z}_i \right \|^2 \notag   \\
		& = \frac{n-s_0}{s_{0}(n-1)}\frac{1}{n} \sum_{i=1}^{n} \left\|  \mathbf{z}_i \right \|^2  \notag  \\
		& = \frac{n-s_0}{s_{0}(n-1)}\frac{1}{n} \sum_{i=1}^{n} \left\|  \nabla_{\mathbf{x}} f^{(k)}_{i}(\mathbf{x}_{0}^{(k)}, \mathbf{y}_{0}^{(k)}) -\nabla_{\mathbf{x}} f^{(k)}(\mathbf{x}_{0}^{(k)}, \mathbf{y}_{0}^{(k)}) \right \|^2  \notag  \\
		& \leq  \frac{n-s_0}{s_{0}(n-1)}\frac{1}{n} \sum_{i=1}^{n} \left\|   \nabla_{\mathbf{x}} f^{(k)}_{i}(\mathbf{x}_{0}^{(k)}, \mathbf{y}_{0}^{(k)})  \right \|^2  \  , 
	\end{align}
	where $ \mathbb{I}_{i\in \mathcal{S}_0} =1$ when $i\in  \mathcal{S}_0$, and $ \mathbb{I}_{i\in \mathcal{S}_0, j\in \mathcal{S}_0}=1$ when $i\in  \mathcal{S}_0$ and $j\in  \mathcal{S}_0$. 
	
	For $t>0$, we can get
	\begin{align}
		& \quad \mathbb{E}[\|\mathbf{v}_{t+1}^{(k)}-\nabla_{\mathbf{x}} f^{(k)}(\mathbf{x}_{t+1}^{(k)}, \mathbf{y}_{t+1}^{(k)})\|^2] \notag \\
		& = \mathbb{E}[\|\frac{1}{s_{t+1}} \sum_{i \in \mathcal{S}_{t+1}}\Big(\nabla_{\mathbf{x}} f^{(k)}_{i}(\mathbf{x}_{t+1}^{(k)}, \mathbf{y}_{t+1}^{(k)})-\nabla_{\mathbf{x}} f^{(k)}_{i}(\mathbf{x}_{t}^{(k)}, \mathbf{y}_{t}^{(k)})\Big)+(1-\rho_{t+1}) \mathbf{v}_{t}^{(k)}\notag \\
		& \quad +\rho_{t+1}\Big(\frac{1}{s_{t+1}} \sum_{i \in \mathcal{S}_{t+1}}(\nabla_{\mathbf{x}} f^{(k)}_{i}(\mathbf{x}_{t}^{(k)}, \mathbf{y}_{t}^{(k)})-\mathbf{g}_{i, t}^{(k)})+\frac{1}{n} \sum_{j=1}^{n} \mathbf{g}_{j, t}^{(k)}\Big) - \nabla_{\mathbf{x}} f^{(k)}(\mathbf{x}_{t+1}^{(k)}, \mathbf{y}_{t+1}^{(k)})\|^2 ] \notag \\
		& = \mathbb{E}[\|\frac{1}{s_{t+1}} \sum_{i \in \mathcal{S}_{t+1}}\Big(\nabla_{\mathbf{x}} f^{(k)}_{i}(\mathbf{x}_{t+1}^{(k)}, \mathbf{y}_{t+1}^{(k)})-\nabla_{\mathbf{x}} f^{(k)}_{i}(\mathbf{x}_{t}^{(k)}, \mathbf{y}_{t}^{(k)}) - \nabla_{\mathbf{x}} f^{(k)}(\mathbf{x}_{t+1}^{(k)}, \mathbf{y}_{t+1}^{(k)})+ \nabla_{\mathbf{x}} f^{(k)}(\mathbf{x}_{t}^{(k)}, \mathbf{y}_{t}^{(k)})\Big)\notag \\
		& \quad +\rho_{t+1}\Big(\frac{1}{s_{t+1}} \sum_{i \in \mathcal{S}_{t+1}}(\nabla_{\mathbf{x}} f^{(k)}_{i}(\mathbf{x}_{t}^{(k)}, \mathbf{y}_{t}^{(k)})-\mathbf{g}_{i, t}^{(k)})+\frac{1}{n} \sum_{j=1}^{n} \mathbf{g}_{j, t}^{(k)}-\nabla_{\mathbf{x}} f^{(k)}(\mathbf{x}_{t}^{(k)}, \mathbf{y}_{t}^{(k)})\Big)\notag \\
		& \quad +(1-\rho_{t+1})( \mathbf{v}_{t}^{(k)}- \nabla_{\mathbf{x}} f^{(k)}(\mathbf{x}_{t}^{(k)}, \mathbf{y}_{t}^{(k)}))\|^2 ] \notag \\
		& \leq \mathbb{E}[\|\frac{1}{s_{t+1}} \sum_{i \in \mathcal{S}_{t+1}}\Big(\nabla_{\mathbf{x}} f^{(k)}_{i}(\mathbf{x}_{t+1}^{(k)}, \mathbf{y}_{t+1}^{(k)})-\nabla_{\mathbf{x}} f^{(k)}_{i}(\mathbf{x}_{t}^{(k)}, \mathbf{y}_{t}^{(k)}) - \nabla_{\mathbf{x}} f^{(k)}(\mathbf{x}_{t+1}^{(k)}, \mathbf{y}_{t+1}^{(k)})+ \nabla_{\mathbf{x}} f^{(k)}(\mathbf{x}_{t}^{(k)}, \mathbf{y}_{t}^{(k)})\Big)\notag \\
		& \quad +\rho_{t+1}\Big(\frac{1}{s_{t+1}} \sum_{i \in \mathcal{S}_{t+1}}(\nabla_{\mathbf{x}} f^{(k)}_{i}(\mathbf{x}_{t}^{(k)}, \mathbf{y}_{t}^{(k)})-\mathbf{g}_{i, t}^{(k)})+\frac{1}{n} \sum_{j=1}^{n} \mathbf{g}_{j, t}^{(k)}-\nabla_{\mathbf{x}} f^{(k)}(\mathbf{x}_{t}^{(k)}, \mathbf{y}_{t}^{(k)})\Big)\|^2]\notag \\
		& \quad +(1-\rho_{t+1})^2\mathbb{E}[\| \mathbf{v}_{t}^{(k)}- \nabla_{\mathbf{x}} f^{(k)}(\mathbf{x}_{t}^{(k)}, \mathbf{y}_{t}^{(k)})\|^2 ] \notag \\
		& \leq 2\mathbb{E}[\|\frac{1}{s_{t+1}} \sum_{i \in \mathcal{S}_{t+1}}\Big(\nabla_{\mathbf{x}} f^{(k)}_{i}(\mathbf{x}_{t+1}^{(k)}, \mathbf{y}_{t+1}^{(k)})-\nabla_{\mathbf{x}} f^{(k)}_{i}(\mathbf{x}_{t}^{(k)}, \mathbf{y}_{t}^{(k)}) - \nabla_{\mathbf{x}} f^{(k)}(\mathbf{x}_{t+1}^{(k)}, \mathbf{y}_{t+1}^{(k)})+ \nabla_{\mathbf{x}} f^{(k)}(\mathbf{x}_{t}^{(k)}, \mathbf{y}_{t}^{(k)})\Big)\|^2]\notag \\
		& \quad +2\rho_{t+1}^2\mathbb{E}[\|\Big(\frac{1}{s_{t+1}} \sum_{i \in \mathcal{S}_{t+1}}(\nabla_{\mathbf{x}} f^{(k)}_{i}(\mathbf{x}_{t}^{(k)}, \mathbf{y}_{t}^{(k)})-\mathbf{g}_{i, t}^{(k)})+\frac{1}{n} \sum_{j=1}^{n} \mathbf{g}_{j, t}^{(k)}-\nabla_{\mathbf{x}} f^{(k)}(\mathbf{x}_{t}^{(k)}, \mathbf{y}_{t}^{(k)})\Big)\|^2]\notag \\
		& \quad +(1-\rho_{t+1})^2\mathbb{E}[\| \mathbf{v}_{t}^{(k)}- \nabla_{\mathbf{x}} f^{(k)}(\mathbf{x}_{t}^{(k)}, \mathbf{y}_{t}^{(k)})\|^2]  \notag \\
		& \leq \frac{2}{s_{t+1}^2}\sum_{i \in \mathcal{S}_{t+1}}\mathbb{E}[\|\nabla_{\mathbf{x}} f^{(k)}_{i}(\mathbf{x}_{t+1}^{(k)}, \mathbf{y}_{t+1}^{(k)})-\nabla_{\mathbf{x}} f^{(k)}_{i}(\mathbf{x}_{t}^{(k)}, \mathbf{y}_{t}^{(k)})- \nabla_{\mathbf{x}} f^{(k)}(\mathbf{x}_{t+1}^{(k)}, \mathbf{y}_{t+1}^{(k)})+ \nabla_{\mathbf{x}} f^{(k)}(\mathbf{x}_{t}^{(k)}, \mathbf{y}_{t}^{(k)})\|^2 ]\notag \\
		& \quad +\frac{ 2\rho_{t+1}^2}{s_{t+1}^2}\sum_{i \in \mathcal{S}_{t+1}} \mathbb{E}[\|\nabla_{\mathbf{x}} f^{(k)}_{i}(\mathbf{x}_{t}^{(k)}, \mathbf{y}_{t}^{(k)})-\mathbf{g}_{i, t}^{(k)}+\frac{1}{n} \sum_{j=1}^{n} \mathbf{g}_{j, t}^{(k)}-\nabla_{\mathbf{x}} f^{(k)}(\mathbf{x}_{t}^{(k)}, \mathbf{y}_{t}^{(k)})\|^2 ]\notag \\
		& \quad +(1-\rho_{t+1})^2\mathbb{E}[\| \mathbf{v}_{t}^{(k)}- \nabla_{\mathbf{x}} f^{(k)}(\mathbf{x}_{t}^{(k)}, \mathbf{y}_{t}^{(k)})\|^2]  \notag \\
		& \leq \frac{2}{s_{t+1}^2}\sum_{i \in \mathcal{S}_{t+1}}\mathbb{E}[\|\nabla_{\mathbf{x}} f^{(k)}_{i}(\mathbf{x}_{t+1}^{(k)}, \mathbf{y}_{t+1}^{(k)})-\nabla_{\mathbf{x}} f^{(k)}_{i}(\mathbf{x}_{t}^{(k)}, \mathbf{y}_{t}^{(k)})\|^2] +\frac{ 2\rho_{t+1}^2}{s_{t+1}^2}\sum_{i \in \mathcal{S}_{t+1}} \mathbb{E}[\|\nabla_{\mathbf{x}} f^{(k)}_{i}(\mathbf{x}_{t}^{(k)}, \mathbf{y}_{t}^{(k)})-\mathbf{g}_{i, t}^{(k)}\|^2] \notag \\
		& \quad +(1-\rho_{t+1})^2\mathbb{E}[\| \mathbf{v}_{t}^{(k)}- \nabla_{\mathbf{x}} f^{(k)}(\mathbf{x}_{t}^{(k)}, \mathbf{y}_{t}^{(k)})\|^2]  \notag \\
		& \leq \frac{2L^2}{s_{t+1}}\mathbb{E}[\|\mathbf{x}_{t+1}^{(k)}- \mathbf{x}_{t}^{(k)}\|^2] + \frac{2L^2}{s_{t+1}}\mathbb{E}[\|\mathbf{y}_{t+1}^{(k)}-\mathbf{y}_{t}^{(k)}\|^2 ]+\frac{ 2\rho_{t+1}^2}{s_{t+1}}\frac{1}{n}\sum_{j=1}^{n}\mathbb{E}[ \|\nabla_{\mathbf{x}} f^{(k)}_{j}(\mathbf{x}_{t}^{(k)}, \mathbf{y}_{t}^{(k)})-\mathbf{g}_{j, t}^{(k)}\|^2] \notag \\
		& \quad +(1-\rho_{t+1})^2\mathbb{E}[\| \mathbf{v}_{t}^{(k)}- \nabla_{\mathbf{x}} f^{(k)}(\mathbf{x}_{t}^{(k)}, \mathbf{y}_{t}^{(k)})\|^2]   \ , 
	\end{align}
	where the second to last step follows from Eq.~(\ref{eq_var}),  the last step follows from Assumption~\ref{assumption_smooth}. 
	\begin{align} \label{eq_var}
		& \quad \mathbb{E}[\|\nabla_{\mathbf{x}} f^{(k)}_{i}(\mathbf{x}_{t}^{(k)}, \mathbf{y}_{t}^{(k)})-\mathbf{g}_{i, t}^{(k)}+\frac{1}{n} \sum_{j=1}^{n} \mathbf{g}_{j, t}^{(k)}-\nabla_{\mathbf{x}} f^{(k)}(\mathbf{x}_{t}^{(k)}, \mathbf{y}_{t}^{(k)})\|^2]  \notag  \\
		&  \leq \mathbb{E}[\|\nabla_{\mathbf{x}} f^{(k)}_{i}(\mathbf{x}_{t}^{(k)}, \mathbf{y}_{t}^{(k)})-\mathbf{g}_{i, t}^{(k)}\|^2 ]\notag \\
		& = \frac{1}{n}\sum_{j=1}^{n} \mathbb{E}[\|\nabla_{\mathbf{x}} f^{(k)}_{j}(\mathbf{x}_{t}^{(k)}, \mathbf{y}_{t}^{(k)})-\mathbf{g}_{j, t}^{(k)}\|^2 ] \ . 
	\end{align}
	Finally, by setting  the batch size to $s_{t}=s_1$ and $\rho_{t}=\rho_1$ for $t>1$, we complete the proof.

\end{proof}

\begin{lemma} \label{lemma_var_y}
	Given Assumption~\ref{graph}-\ref{assumption_strong}, for $t>0$,  by setting $s_{t}=s_1$ and $\rho_{t}=\rho_1$, we have
	\begin{align}
		& \quad \mathbb{E}[\|\mathbf{u}_{t+1}^{(k)}-\nabla_{\mathbf{y}} f^{(k)}(\mathbf{x}_{t+1}^{(k)}, \mathbf{y}_{t+1}^{(k)})\|^2 ] \leq (1-\rho_{1})^2\mathbb{E}[\| \mathbf{u}_{t}^{(k)}- \nabla_{\mathbf{y}} f^{(k)}(\mathbf{x}_{t}^{(k)}, \mathbf{y}_{t}^{(k)})\|^2 ]  \notag \\
		& \leq \frac{2L^2}{s_{1}}\mathbb{E}[\|\mathbf{x}_{t+1}^{(k)}- \mathbf{x}_{t}^{(k)}\|^2 ]+ \frac{2L^2}{s_{1}}\mathbb{E}[\|\mathbf{y}_{t+1}^{(k)}-\mathbf{y}_{t}^{(k)}\|^2+\frac{ 2\rho_{1}^2}{s_{1}}\frac{1}{n}\sum_{i=1}^{n} \mathbb{E}[\|\nabla_{\mathbf{y}} f^{(k)}_{i}(\mathbf{x}_{t}^{(k)}, \mathbf{y}_{t}^{(k)})-\mathbf{h}_{i, t}^{(k)}\|^2 ] \ . 
	\end{align}

	For $t=0$,  $\rho_{0}=1$,  we have
	\begin{align}
		&  \mathbb{E}[\|\mathbf{u}_{0}^{(k)}-\nabla_{\mathbf{y}} f^{(k)}(\mathbf{x}_{0}^{(k)}, \mathbf{y}_{0}^{(k)})\|^2 ]\leq  \frac{n-s_0}{(n-1)s_0}\frac{1}{n}\sum_{i=1}^{n} \|\nabla_{\mathbf{y}} f^{(k)}_i(\mathbf{x}_{0}^{(k)}, \mathbf{y}_{0}^{(k)})\|^2 \ .
	\end{align}
\end{lemma} 
Lemma~\ref{lemma_var_y} can be proved by following Lemma~\ref{lemma_var_x}. Thus, we ignore it. 

\begin{lemma}  \label{lemma_var_g_x}
	Given Assumption~\ref{graph}-\ref{assumption_strong}, for $t>0$, by setting $s_{t}=s_1$ and $\alpha_{t}=\alpha_1>0$, we have
	\begin{align}
		&  \mathbb{E}[\frac{1}{n} \sum_{j=1}^{n}\|\nabla_{\mathbf{x}} f^{(k)}_{j}(\mathbf{x}_{t+1}^{(k)}, \mathbf{y}_{t+1}^{(k)})-\mathbf{g}_{j,t+1}^{(k)}\|^{2}] \leq (1-\frac{s_{1}}{n})(1+\alpha_{1})\frac{1}{n}\sum_{j=1}^{n}\| \nabla_{\mathbf{x}} f^{(k)}_{j}(\mathbf{x}_{t}^{(k)}, \mathbf{y}_{t}^{(k)})-\mathbf{g}_{j,t}^{(k)}\|^{2} \notag \\
		&\quad  +  2L^2(1-\frac{s_{1}}{n}) (1+\frac{1}{\alpha_{1}})\|\mathbf{x}_{t+1}^{(k)}-\mathbf{x}_{t}^{(k)}\|^2 + 2L^2(1-\frac{s_{1}}{n}) (1+\frac{1}{\alpha_{1}})\| \mathbf{y}_{t+1}^{(k)}- \mathbf{y}_{t}^{(k)}\|^2 \ . 
	\end{align}
	For $t=0$, we have
	\begin{align}
		& \mathbb{E}[\frac{1}{n} \sum_{j=1}^{n}\|\nabla_{\mathbf{x}} f^{(k)}_{j}(\mathbf{x}_{0}^{(k)}, \mathbf{y}_{0}^{(k)})-\mathbf{g}_{j,0}^{(k)}\|^{2}]  \leq  (1-\frac{s_{0}}{n}) \frac{1}{n}\sum_{j=1}^{n}\|\nabla_{\mathbf{x}} f^{(k)}_{j}(\mathbf{x}_{0}^{(k)}, \mathbf{y}_{0}^{(k)})\|^{2} \ . 
	\end{align}

\end{lemma}

\begin{proof}
	For $t>0$,  we have
	\begin{align}
		&\quad  \mathbb{E}[\frac{1}{n} \sum_{j=1}^{n}\|\nabla_{\mathbf{x}} f^{(k)}_{j}(\mathbf{x}_{t+1}^{(k)}, \mathbf{y}_{t+1}^{(k)})-\mathbf{g}_{j,t+1}^{(k)}\|^{2}]  \notag \\
		& = (1-\frac{s_{t+1}}{n}) \frac{1}{n}\sum_{j=1}^{n} \mathbb{E}[\|\nabla_{\mathbf{x}} f^{(k)}_{j}(\mathbf{x}_{t+1}^{(k)}, \mathbf{y}_{t+1}^{(k)})-\mathbf{g}_{j,t}^{(k)}\|^{2}] \notag \\
		& = (1-\frac{s_{t+1}}{n}) \frac{1}{n}\sum_{j=1}^{n} \mathbb{E}[\|\nabla_{\mathbf{x}} f^{(k)}_{j}(\mathbf{x}_{t+1}^{(k)}, \mathbf{y}_{t+1}^{(k)})-\nabla_{\mathbf{x}} f^{(k)}_{j}(\mathbf{x}_{t}^{(k)}, \mathbf{y}_{t}^{(k)}) + \nabla_{\mathbf{x}} f^{(k)}_{j}(\mathbf{x}_{t}^{(k)}, \mathbf{y}_{t}^{(k)})-\mathbf{g}_{j,t}^{(k)}\|^{2} ]\notag \\
		& \leq 2L^2(1-\frac{s_{t+1}}{n}) (1+\frac{1}{\alpha_{t+1}}) \mathbb{E}[\|\mathbf{x}_{t+1}^{(k)}-\mathbf{x}_{t}^{(k)}\|^2] + 2L^2(1-\frac{s_{t+1}}{n}) (1+\frac{1}{\alpha_{t+1}}) \mathbb{E}[\| \mathbf{y}_{t+1}^{(k)}- \mathbf{y}_{t}^{(k)}\|^2]\notag \\
		& \quad +(1-\frac{s_{t+1}}{n})(1+\alpha_{t+1})\frac{1}{n}\sum_{j=1}^{n} \mathbb{E}[\| \nabla_{\mathbf{x}} f^{(k)}_{j}(\mathbf{x}_{t}^{(k)}, \mathbf{y}_{t}^{(k)})-\mathbf{g}_{j,t}^{(k)}\|^{2} ]  \ , 
	\end{align}
	where $\alpha_{t+1}>0$. By setting $s_{t}=s_1$ and $\alpha_{t}=\alpha_1$, we complete the proof for the first part. 
	
	For $t=0$, we have
	\begin{align}
		& \mathbb{E}[\frac{1}{n} \sum_{j=1}^{n}\|\nabla_{\mathbf{x}} f^{(k)}_{j}(\mathbf{x}_{0}^{(k)}, \mathbf{y}_{0}^{(k)})-\mathbf{g}_{j,0}^{(k)}\|^{2}]   = (1-\frac{s_{0}}{n}) \frac{1}{n}\sum_{j=1}^{n}\|\nabla_{\mathbf{x}} f^{(k)}_{j}(\mathbf{x}_{0}^{(k)}, \mathbf{y}_{0}^{(k)})\|^{2}  \ . 
	\end{align}
\end{proof}

\begin{lemma}  \label{lemma_var_g_y}
	Given Assumption~\ref{graph}-\ref{assumption_strong}, for $t>0$, by setting $s_{t}=s_1$ and $\alpha_{t}=\alpha_1$, we have
	\begin{align}
		& \mathbb{E}[\frac{1}{n} \sum_{j=1}^{n}\|\nabla_{\mathbf{y}} f^{(k)}_{j}(\mathbf{x}_{t+1}^{(k)}, \mathbf{y}_{t+1}^{(k)})-\mathbf{h}_{j,t+1}^{(k)}\|^{2}] \leq  (1-\frac{s_{1}}{n})(1+\alpha_{t})\frac{1}{n}\sum_{j=1}^{n}\| \nabla_{\mathbf{y}} f^{(k)}_{j}(\mathbf{x}_{t}^{(k)}, \mathbf{y}_{t}^{(k)})-\mathbf{h}_{j,t}^{(k)}\|^{2}  \notag \\
		&\quad  + 2L^2(1-\frac{s_{1}}{n}) (1+\frac{1}{\alpha_{t}})\|\mathbf{x}_{t+1}^{(k)}-\mathbf{x}_{t}^{(k)}\|^2  + 2L^2(1-\frac{s_{1}}{n}) (1+\frac{1}{\alpha_{t}})\| \mathbf{y}_{t+1}^{(k)}- \mathbf{y}_{t}^{(k)}\|^2 \ . 
	\end{align}
	For $t=0$, we have
	\begin{align}
		& \mathbb{E}[\frac{1}{n} \sum_{j=1}^{n}\|\nabla_{\mathbf{y}} f^{(k)}_{j}(\mathbf{x}_{0}^{(k)}, \mathbf{y}_{0}^{(k)})-\mathbf{h}_{j,0}\|^{2}]\leq  (1-\frac{s_{0}}{n}) \frac{1}{n}\sum_{j=1}^{n}\|\nabla_{\mathbf{y}} f^{(k)}_{j}(\mathbf{x}_{0}^{(k)}, \mathbf{y}_{0}^{(k)})\|^{2}  \ . 
	\end{align}
\end{lemma}
Lemma~\ref{lemma_var_g_y} can be proved by following Lemma~\ref{lemma_var_g_x}. Thus, we do not include it.

\begin{lemma} \label{lemma_consensus_x}
	Given Assumption~\ref{graph}-\ref{assumption_strong},  we have 
	
	\begin{align}
		&   \sum_{k=1}^{K} \|\mathbf{x}_{t+1}^{(k)}-\bar{\mathbf{x}}_{t+1}\|^2\leq  \Big(1-\frac{\eta(1-\lambda^2)}{2}\Big)\sum_{k=1}^{K}\|\mathbf{x}_{t}^{(k)}-\bar{\mathbf{x}}_{t} \|^2+ \frac{2\eta\gamma_1^2}{1-\lambda^2}\sum_{k=1}^{K}\|\mathbf{a}_{t}^{(k)}-\bar{\mathbf{a}}_{t} \|^2 \ . 
	\end{align}
	
\end{lemma}

\begin{proof}
	
	\begin{align}
		&  \quad \sum_{k=1}^{K} \|\mathbf{x}_{t+1}^{(k)}-\bar{\mathbf{x}}_{t+1}\|^2 = \|X_{t+1} - \bar{X}_{t+1}\|_F^2 \notag \\
		& = \|(1-\eta)X_{t} + \eta X_{t+\frac{1}{2}} - (1-\eta)\bar{X}_{t} - \eta \bar{X}_{t+\frac{1}{2}} \|_F^2 \notag \\
		& \leq  (1+a_0)(1-\eta)^2\|X_t -\bar{X}_t\|_F^2+ \eta^2(1+\frac{1}{a_0})\|{X}_{t+\frac{1}{2}} -\bar{X}_{t+\frac{1}{2}} \|_F^2\notag \\
		& \leq (1-\eta) \|X_t -\bar{X}_t\|_F^2+ \eta\|{X}_{t+\frac{1}{2}} -\bar{X}_{t+\frac{1}{2}} \|_F^2\notag \\
		& \leq (1-\eta)\|X_t -\bar{X}_t\|_F^2+ \eta\|X_{t}W+\gamma_1A_t - (\bar{X}_t + \gamma_1\bar{A}_t)\|_F^2\notag \\ 
		& \leq (1-\eta)\|X_t -\bar{X}_t\|_F^2+ \eta(1+a_1)\|X_{t}W  -\bar{X}_{t} \|_F^2+ \eta\gamma_1^2(1+\frac{1}{a_1})\|A_t - \bar{A}_t \|_F^2\notag \\ 
		& \leq (1-\eta)\|X_t -\bar{X}_t\|_F^2+ \frac{\eta(1+\lambda^2)}{2}\|X_{t}  -\bar{X}_{t} \|_F^2+ \frac{2\eta\gamma_1^2}{1-\lambda^2}\|A_t - \bar{A}_t \|_F^2\notag \\ 
		& = \Big(1-\eta+\frac{\eta(1+\lambda^2)}{2}\Big)\|X_{t}  -\bar{X}_{t} \|_F^2+ \frac{2\eta\gamma_1^2}{1-\lambda^2}\|A_t - \bar{A}_t \|_F^2\notag \\ 
		& = \Big(1-\frac{\eta(1-\lambda^2)}{2}\Big)\sum_{k=1}^{K}\|\mathbf{x}_{t}^{(k)}-\bar{\mathbf{x}}_{t} \|^2+ \frac{2\eta\gamma_1^2}{1-\lambda^2}\sum_{k=1}^{K}\|\mathbf{a}_{t}^{(k)}-\bar{\mathbf{a}}_{t} \|^2 \ , 
	\end{align}
	where the second inequality follows from $a_0=\frac{\eta}{1-\eta}$, the last inequality follows from $a_1=\frac{1-\lambda^2}{2\lambda^2}$ and $\|X_{t}W  -\bar{X}_{t} \|_F^2\leq \lambda^2 \|X_{t}  -\bar{X}_{t} \|_F^2$.

\end{proof}

\begin{lemma} \label{lemma_consensus_y}
	Given Assumption~\ref{graph}-\ref{assumption_strong},  we have 
	\begin{align}
		&   \sum_{k=1}^{K} \|\mathbf{y}_{t+1}^{(k)}-\bar{\mathbf{y}}_{t+1}\|^2\leq  \Big(1-\frac{\eta(1-\lambda^2)}{2}\Big)\sum_{k=1}^{K}\|\mathbf{y}_{t}^{(k)}-\bar{\mathbf{y}}_{t} \|^2+ \frac{2\eta\gamma_2^2}{1-\lambda^2}\sum_{k=1}^{K}\|\mathbf{b}_{t}^{(k)}-\bar{\mathbf{b}}_{t} \|^2 \ . 
	\end{align}
	
\end{lemma}
Similarly, we can prove the second inequality regarding $\mathbf{y}$ in Lemma~\ref{lemma_consensus_x}.

\begin{lemma} \label{lemma_incremental}
	Given Assumption~\ref{graph}-\ref{assumption_strong}, we have
	\begin{align}
		&  \quad \|X_{t+1} - X_{t}\|_F^2\leq 12 \eta^2\sum_{k=1}^{K}\|\mathbf{x}_{t}^{(k)} - \bar{\mathbf{x}}_{t} \|^2 + 3  \gamma_1^2\eta^2\sum_{k=1}^{K} \|\mathbf{a}_{t}^{(k)} - \bar{\mathbf{a}}_{t}\|^2 + 3\gamma_1^2\eta^2K \|\bar{\mathbf{v}}_{t}\|^2 \ ,  \\
		&  \quad \|Y_{t+1} - Y_{t}\|_F^2  \leq 12 \eta^2\sum_{k=1}^{K}\|\mathbf{y}_{t}^{(k)} - \bar{\mathbf{y}}_{t} \|^2 + 3  \gamma_2^2\eta^2\sum_{k=1}^{K} \|\mathbf{b}_{t}^{(k)} - \bar{\mathbf{b}}_{t}\|^2 + 3\gamma_2^2\eta^2K \|\bar{\mathbf{u}}_{t}\|^2 \ . 
	\end{align}
	
\end{lemma}

\begin{proof}
	
	\begin{align}
		&  \quad \|X_{t+1} - X_{t}\|_F^2 \notag \\
		& = \eta^2\|X_{t+\frac{1}{2}} - X_{t}\|_F^2 \notag \\
		& =  \eta^2\|X_{t}W - \gamma_1 A_t - X_{t}\|_F^2 \notag \\
		& =  \eta^2\|X_{t}W - X_{t} - \gamma_1 A_t + \gamma_1 \bar{A}_t - \gamma_1 \bar{A}_t \|_F^2 \notag \\
		& \leq 3 \eta^2\|X_{t}W - X_{t}\|_F^2 + 3  \gamma_1^2\eta^2 \|A_t - \bar{A}_t\|_F^2 + 3\gamma_1^2\eta^2\|\bar{A}_t\|_F^2 \notag \\
		& = 3 \eta^2\|(X_{t} -\bar{X}_{t})(W-I)\|_F^2 + 3  \gamma_1^2\eta^2 \|A_t - \bar{A}_t\|_F^2 + 3\gamma_1^2\eta^2\|\bar{A}_t\|_F^2 \notag \\
		& \leq 12 \eta^2\|X_{t} -\bar{X}_{t}\|_F^2  + 3  \gamma_1^2\eta^2 \|A_t - \bar{A}_t\|_F^2 + 3\gamma_1^2\eta^2\|\bar{V}_t\|_F^2 \notag \\
		& = 12 \eta^2\sum_{k=1}^{K}\|\mathbf{x}_{t}^{(k)} - \bar{\mathbf{x}}_{t} \|^2 + 3  \gamma_1^2\eta^2\sum_{k=1}^{K} \|\mathbf{a}_{t}^{(k)} - \bar{\mathbf{a}}_{t}\|^2 + 3\gamma_1^2\eta^2K \|\bar{\mathbf{v}}_{t}\|^2 \ , 
	\end{align}
	where the last inequality follows from $\frac{1}{K}\sum_{k=1}^{K}\mathbf{a}_t^{(k)} = \frac{1}{K}\sum_{k=1}^{K}\mathbf{v}_t^{(k)}$.  Similarly, we can prove the inequality for $\|Y_{t+1} - Y_{t}\|_F^2$.

\end{proof}

Based on these lemmas, we prove Theorem~\ref{theorem} in the following. 
\begin{proof}
	For $t\geq1$, we define the potential function  as follows:
	\begin{align}
		& H_{t}=\mathbb{E}[\Phi(\bar{\mathbf{x}}_{t})]+ C_0 \mathbb{E}[\|\bar{\mathbf{y}}_{t}   - \mathbf{y}^{*}(\bar{\mathbf{x}}_t)\| ^2 ]  \notag \\
		& \quad +  \frac{C_1}{K}\sum_{k=1}^{K}\mathbb{E}[\|\nabla_{\mathbf{x}} f^{(k)}(\mathbf{x}^{(k)}_t, \mathbf{y}^{(k)}_t)-\mathbf{v}_t^{(k)}\|^2] +  \frac{C_2}{K}\sum_{k=1}^{K}\mathbb{E}[\|\nabla_{\mathbf{y}} f^{(k)}(\mathbf{x}^{(k)}_t, \mathbf{y}^{(k)}_t)-\mathbf{u}_t^{(k)}\|^2] \notag \\
		& \quad + \frac{C_3}{K} \sum_{k=1}^{K}\mathbb{E}[\frac{1}{n} \sum_{j=1}^{n}\|\nabla_{\mathbf{x}} f^{(k)}_{j}(\mathbf{x}_{t}^{(k)}, \mathbf{y}_{t}^{(k)})-\mathbf{g}_{j,t}^{(k)}\|^{2}]  + \frac{C_4}{K}\sum_{k=1}^{K}\mathbb{E}[\frac{1}{n} \sum_{j=1}^{n}\|\nabla_{\mathbf{y}} f^{(k)}_{j}(\mathbf{x}_{t}^{(k)}, \mathbf{y}_{t}^{(k)})-\mathbf{h}_{j,t}^{(k)}\|^{2}] \notag \\
		& \quad + \frac{C_5}{K}\sum_{k=1}^{K} \mathbb{E}[\|\bar{\mathbf{x}}_t-\mathbf{x}^{(k)}_{t}\|^2]+ \frac{C_6}{K}\sum_{k=1}^{K} \mathbb{E}[\|\bar{\mathbf{y}}_t-\mathbf{y}^{(k)}_{t}\|^2] \notag \\
		& \quad  + \frac{C_7}{K}\sum_{k=1}^{K} \mathbb{E}[\|\bar{\mathbf{a}}_t-\mathbf{a}^{(k)}_{t}\|^2]+ \frac{C_8}{K}\sum_{k=1}^{K}\mathbb{E}[ \|\bar{\mathbf{b}}_t-\mathbf{b}^{(k)}_{t}\|^2 ] \  . 
	\end{align}

	Then, according to Lemma~\ref{lemma_fx},~\ref{lemma_y},~\ref{lemma_var_x},~\ref{lemma_var_y},~\ref{lemma_var_g_x},~\ref{lemma_var_g_y},~\ref{lemma_consensus_x},~\ref{lemma_consensus_y},~\ref{lemma_consensus_a},~\ref{lemma_consensus_b}, it is easy to get
	\begin{align}
		& \quad H_{t+1} - H_{t}  \leq - \frac{\gamma_1\eta}{2} \mathbb{E}[\|\nabla \Phi(\bar{\mathbf{x}}_{t})\|^2] + \Big( \gamma_1\eta L^2 -\frac{\eta\gamma_2\mu}{4}C_0 \Big)  \mathbb{E}[\|\bar{\mathbf{y}}_t -  \mathbf{y}^*(\bar{\mathbf{x}}_t)\|^2]\notag \\
		& \quad + \Big( \frac{25\eta\gamma_1^2 \kappa^2 }{6\gamma_2\mu}C_0 -  \frac{\gamma_1\eta}{4} \Big)\mathbb{E}[\|\bar{\mathbf{v}}_t\|^2]    +\Big(- \frac{3\gamma_2^2\eta}{4}C_0 \Big) \mathbb{E}[\|\bar{\mathbf{u}}_t\|^2]    \notag \\
		& \quad  + \Big(2\gamma_1\eta L^2+  \frac{25\eta \gamma_2 L^2}{3\mu}C_0 -\frac{\eta(1-\lambda^2)}{2}C_5\Big) \frac{1}{K}\sum_{k=1}^{K} \mathbb{E}[\|\bar{\mathbf{x}}_t-\mathbf{x}^{(k)}_{t}\|^2]  \notag \\
		& \quad + \Big(2\gamma_1\eta L^2+  \frac{25\eta \gamma_2 L^2}{3\mu} C_0-\frac{\eta(1-\lambda^2)}{2} C_6\Big) \frac{1}{K}\sum_{k=1}^{K} \mathbb{E}[\|\bar{\mathbf{y}}_t-\mathbf{y}^{(k)}_{t}\|^2]  \notag \\
		& \quad + \Big(2\gamma_1\eta+C_1(1-\rho_{1})^2+  \frac{6\rho_{1}^2 }{1-\lambda^2}C_7-C_1\Big)\frac{1}{K}\sum_{k=1}^{K}\mathbb{E}[\|\nabla_{\mathbf{x}} f^{(k)}(\mathbf{x}^{(k)}_t, \mathbf{y}^{(k)}_t) -\mathbf{v}^{(k)}_t\|^2]\notag \\
		& \quad    + \Big(\frac{25\eta \gamma_2 }{3\mu}C_0 + C_2(1-\rho_{1})^2+  \frac{6\rho_{1}^2 }{1-\lambda^2}C_8-C_2\Big) \frac{1}{K}\sum_{k=1}^{K}\mathbb{E}[\|\nabla_{\mathbf{y}} f^{(k)}({\mathbf{x}}^{(k)}_t, {\mathbf{y}}^{(k)}_t)  -{\mathbf{u}}^{(k)}_t\|^2]   \notag \\
		& \quad +\Big(C_3(1-\frac{s_{1}}{n})(1+\alpha_{1})+\frac{ 2\rho_{1}^2}{s_{1}} C_1+ \frac{6\rho_{1}^2}{(1-\lambda^2)s_1} C_7-C_3\Big)\frac{1}{K}\sum_{k=1}^{K}\frac{1}{n}\sum_{i=1}^{n} \mathbb{E}[\|\nabla_{\mathbf{x}} f^{(k)}_{i}(\mathbf{x}_{t}^{(k)}, \mathbf{y}_{t}^{(k)})-\mathbf{g}_{i, t}^{(k)}\|^2] \notag \\
		&\quad  +\Big(C_4(1-\frac{s_{1}}{n})(1+\alpha_{1})+\frac{ 2\rho_{1}^2}{s_{1}}C_2 +\frac{6\rho_{1}^2}{(1-\lambda^2)s_1} C_8-C_4\Big) \frac{1}{K}\sum_{k=1}^{K}\frac{1}{n}\sum_{i=1}^{n} \mathbb{E}[\|\nabla_{\mathbf{y}} f^{(k)}_{i}(\mathbf{x}_{t}^{(k)}, \mathbf{y}_{t}^{(k)})-\mathbf{h}_{i, t}^{(k)}\|^2] \notag \\
		& \quad + \Big(\frac{2\gamma_1^2\eta}{1-\lambda^2}C_5- \frac{1-\lambda^2}{2}C_7\Big) \frac{1}{K}\sum_{k=1}^{K}\mathbb{E}[\|\mathbf{a}_{t}^{(k)}-\bar{\mathbf{a}}_{t}\|^2]   + \Big(\frac{2\eta\gamma_2^2}{1-\lambda^2}C_6- \frac{1-\lambda^2}{2}C_8\Big)\frac{1}{K}\sum_{k=1}^{K} \mathbb{E}[\|\mathbf{b}_{t}^{(k)}-\bar{\mathbf{b}}_{t}\|^2]  \notag \\
		& \quad + C_{9}\frac{1}{K}\sum_{k=1}^{K}\mathbb{E}[\|\mathbf{x}_{t+1}^{(k)}- \mathbf{x}_{t}^{(k)}\|^2] + C_{9}\frac{1}{K}\sum_{k=1}^{K}\mathbb{E}[\|\mathbf{y}_{t+1}^{(k)}-\mathbf{y}_{t}^{(k)}\|^2]  \ , 
	\end{align}
	where 
	\begin{align}
		& C_{9} = \frac{4L^2}{s_{1}}C_1+\frac{4L^2}{s_{1}}C_2+ 2L^2(1-\frac{s_{1}}{n}) (1+\frac{1}{\alpha_{1}})C_3+ 2L^2(1-\frac{s_{1}}{n}) (1+\frac{1}{\alpha_{1}})C_4+  \frac{6L^2}{(1-\lambda^2)} C_7+  \frac{6L^2}{(1-\lambda^2)}C_8 \ . 
	\end{align}
	
	Then, according to Lemma~\ref{lemma_incremental}, we can get
	\begin{align}
		& \quad H_{t+1} - H_{t}  \leq - \frac{\gamma_1\eta}{2} \mathbb{E}[\|\nabla \Phi(\bar{\mathbf{x}}_{t})\|^2] + \Big( \gamma_1\eta L^2 -\frac{\eta\gamma_2\mu}{4}C_0 \Big)  \mathbb{E}[\|\bar{\mathbf{y}}_t -  \mathbf{y}^*(\bar{\mathbf{x}}_t)\|^2]\notag \\
		& \quad + \Big( \frac{25\eta\gamma_1^2 \kappa^2 }{6\gamma_2\mu}C_0 + 3\gamma_1^2\eta^2 C_{9} -  \frac{\gamma_1\eta}{4} \Big)\mathbb{E}[\|\bar{\mathbf{v}}_t\|^2]    +\Big( 3\gamma_2^2\eta^2C_{9}- \frac{3\gamma_2^2\eta}{4}C_0 \Big) \mathbb{E}[\|\bar{\mathbf{u}}_t\|^2]    \notag \\
		& \quad  + \Big(2\gamma_1\eta L^2+  \frac{25\eta \gamma_2 L^2}{3\mu}C_0 + 12 \eta^2 C_{9}-\frac{\eta(1-\lambda^2)}{2}C_5\Big) \frac{1}{K}\sum_{k=1}^{K} \mathbb{E}[\|\bar{\mathbf{x}}_t-\mathbf{x}^{(k)}_{t}\|^2]  \notag \\
		& \quad + \Big(2\gamma_1\eta L^2+  \frac{25\eta \gamma_2 L^2}{3\mu} C_0+ 12 \eta^2 C_{9}-\frac{\eta(1-\lambda^2)}{2} C_6\Big) \frac{1}{K}\sum_{k=1}^{K} \mathbb{E}[\|\bar{\mathbf{y}}_t-\mathbf{y}^{(k)}_{t}\|^2]  \notag \\
		& \quad + \Big(2\gamma_1\eta+C_1(1-\rho_{1})^2+  \frac{6\rho_{1}^2 }{1-\lambda^2}C_7-C_1\Big)\frac{1}{K}\sum_{k=1}^{K}\mathbb{E}[\|\nabla_{\mathbf{x}} f^{(k)}(\mathbf{x}^{(k)}_t, \mathbf{y}^{(k)}_t) -\mathbf{v}^{(k)}_t\|^2]\notag \\
		& \quad    + \Big(\frac{25\eta \gamma_2 }{3\mu}C_0 + C_2(1-\rho_{1})^2+  \frac{6\rho_{1}^2 }{1-\lambda^2}C_8-C_2\Big) \frac{1}{K}\sum_{k=1}^{K}\mathbb{E}[\|\nabla_{\mathbf{y}} f^{(k)}({\mathbf{x}}^{(k)}_t, {\mathbf{y}}^{(k)}_t)  -{\mathbf{u}}^{(k)}_t\|^2]   \notag \\
		& \quad +\Big(C_3(1-\frac{s_{1}}{n})(1+\alpha_{1})+\frac{ 2\rho_{1}^2}{s_{1}} C_1+ \frac{6\rho_{1}^2}{(1-\lambda^2)s_1} C_7-C_3\Big)\frac{1}{K}\sum_{k=1}^{K}\frac{1}{n}\sum_{i=1}^{n} \mathbb{E}[\|\nabla_{\mathbf{x}} f^{(k)}_{i}(\mathbf{x}_{t}^{(k)}, \mathbf{y}_{t}^{(k)})-\mathbf{g}_{i, t}^{(k)}\|^2] \notag \\
		&\quad  +\Big(C_4(1-\frac{s_{1}}{n})(1+\alpha_{1})+\frac{ 2\rho_{1}^2}{s_{1}}C_2 +\frac{6\rho_{1}^2}{(1-\lambda^2)s_1} C_8-C_4\Big) \frac{1}{K}\sum_{k=1}^{K}\frac{1}{n}\sum_{i=1}^{n} \mathbb{E}[\|\nabla_{\mathbf{y}} f^{(k)}_{i}(\mathbf{x}_{t}^{(k)}, \mathbf{y}_{t}^{(k)})-\mathbf{h}_{i, t}^{(k)}\|^2] \notag \\
		& \quad + \Big(\frac{2\gamma_1^2\eta}{1-\lambda^2}C_5  + 3  \gamma_1^2\eta^2 C_{9}- \frac{1-\lambda^2}{2}C_7\Big) \frac{1}{K}\sum_{k=1}^{K}\mathbb{E}[\|\mathbf{a}_{t}^{(k)}-\bar{\mathbf{a}}_{t}\|^2]  \notag \\
		& \quad  + \Big(\frac{2\eta\gamma_2^2}{1-\lambda^2}C_6+ 3  \gamma_2^2\eta^2C_{9}- \frac{1-\lambda^2}{2}C_8\Big)\frac{1}{K}\sum_{k=1}^{K} \mathbb{E}[\|\mathbf{b}_{t}^{(k)}-\bar{\mathbf{b}}_{t}\|^2]   \ . 
	\end{align}
	
		By setting $C_0 = \frac{6\gamma_1  L^2}{\gamma_2\mu}$, we can get
		\begin{align}
			\gamma_1\eta L^2 -\frac{\eta\gamma_2\mu}{4}C_0 =  - \frac{\gamma_1\eta}{2}  \ . 
		\end{align}
		
		By setting 
		\begin{align}
			& C_1 =  \frac{3\gamma_1\eta}{\rho_{1}}  \ , \notag \\
			& C_2 = \frac{51\eta \gamma_1 L^2}{\rho_{1}\mu^2} \ ,  \notag \\
			& C_7 = \frac{(1-\lambda^2)\gamma_1\eta}{6\rho_{1}}  \ , \notag \\
			&  C_8 =  \frac{(1-\lambda^2)\eta \gamma_1 L^2}{6\rho_{1}\mu^2}   \ , \notag \\
			& \rho_1=\frac{s_1}{2n} < 1 \ , 
		\end{align}
		we can get
		\begin{align}
			&\quad  2\gamma_1\eta+C_1(1-\rho_{1})^2+  \frac{6\rho_{1}^2 }{1-\lambda^2}C_7-C_1 \notag  \\
			& \leq 2\gamma_1\eta+  \frac{6\rho_{1} }{1-\lambda^2}C_7-\rho_{1}C_1     \notag \\
			& \leq 2\gamma_1\eta+  \frac{6\rho_{1}^2 }{1-\lambda^2} \frac{(1-\lambda^2)\gamma_1\eta}{6\rho_{1}} -\rho_{1} \frac{3\gamma_1\eta}{\rho_{1}}     \notag \\
			& \leq  0 \ , 
		\end{align}
		and 
		\begin{align}
			& \quad \frac{25\eta \gamma_2 }{3\mu}C_0 + C_2(1-\rho_{1})^2+  \frac{6\rho_{1}^2 }{1-\lambda^2}C_8-C_2 \notag \\
			& \leq \frac{25\eta \gamma_2 }{3\mu}C_0 +  \frac{6\rho_{1} }{1-\lambda^2}C_8-\rho_{1} C_2 \notag \\
			& \leq \frac{25\eta \gamma_2 }{3\mu} \frac{6\gamma_1  L^2}{\gamma_2\mu} +  \frac{6\rho_{1} }{1-\lambda^2} \frac{(1-\lambda^2)\eta \gamma_1 L^2}{6\rho_{1}\mu^2} -\rho_{1}  \frac{51\eta \gamma_1 L^2}{\rho_{1}\mu^2}  \notag \\
			& \leq  0 \ .
		\end{align}
		
		By setting
		\begin{align}
			& C_3 = \frac{ 14n\rho_{1}\gamma_1\eta}{s_{1}^2}  \ , \notag \\
			& C_4 = \frac{226n\rho_{1}\eta \gamma_1 L^2}{s_1^2\mu^2}  \ , \notag \\
			& \alpha_1 = \frac{s_1}{2n} < 1 \ , 
		\end{align}
		we can get
		\begin{align}
			& \quad C_3(1-\frac{s_{1}}{n})(1+\alpha_{1})+\frac{ 2\rho_{1}^2}{s_{1}} C_1+ \frac{6\rho_{1}^2}{(1-\lambda^2)s_1} C_7-C_3 \notag  \\
			& \leq \frac{ 2\rho_{1}^2}{s_{1}} C_1+ \frac{6\rho_{1}^2}{(1-\lambda^2)s_1} C_7-\frac{s_1}{2n} C_3   \notag  \\
			& \leq \frac{ 2\rho_{1}^2}{s_{1}}  \frac{3\gamma_1\eta}{\rho_{1}}+ \frac{6\rho_{1}^2}{(1-\lambda^2)s_1} \frac{(1-\lambda^2)\gamma_1\eta}{6\rho_{1}}-\frac{s_1}{2n} \frac{ 14n\rho_{1}\gamma_1\eta}{s_{1}^2}    \notag  \\
			& \leq  0 \ ,
		\end{align}
		and 
		\begin{align}
			& \quad C_4(1-\frac{s_{1}}{n})(1+\alpha_{1})+\frac{ 2\rho_{1}^2}{s_{1}}C_2 +\frac{6\rho_{1}^2}{(1-\lambda^2)s_1} C_8-C_4 \notag \\
			& \leq\frac{ 2\rho_{1}^2}{s_{1}}C_2 +\frac{6\rho_{1}^2}{(1-\lambda^2)s_1} C_8-\frac{s_1}{2n} C_4 \notag \\
			& \leq\frac{ 2\rho_{1}^2}{s_{1}} \frac{56\eta \gamma_1 L^2}{\rho_{1}\mu^2}   +\frac{6\rho_{1}^2}{(1-\lambda^2)s_1}  \frac{(1-\lambda^2)\eta \gamma_1 L^2}{6\rho_{1}\mu^2}-\frac{s_1}{2n} \frac{226n\rho_{1}\eta \gamma_1 L^2}{s_1^2\mu^2} \notag \\
			& \leq   0 \ .
		\end{align}
		
		Therefore, based on the following values 
		\begin{align}
			&  C_0 = \frac{6\gamma_1  L^2}{\gamma_2\mu}  \ ,   \quad  C_1 =  \frac{3\gamma_1\eta}{\rho_{1}}  \ ,    \quad C_2 = \frac{51\eta \gamma_1 L^2}{\rho_{1}\mu^2} \ ,  \quad  C_3 = \frac{ 14n\rho_{1}\gamma_1\eta}{s_{1}^2} \ , \quad  C_4 = \frac{226n\rho_{1}\eta \gamma_1 L^2}{s_1^2\mu^2}   \ ,  \notag \\
			& C_7 = \frac{(1-\lambda^2)\gamma_1\eta}{6\rho_{1}}  \ , \quad   C_8 =  \frac{(1-\lambda^2)\eta \gamma_1 L^2}{6\rho_{1}\mu^2}   \ ,  \quad \rho_1=\frac{s_1}{2n}  \ , \quad  \alpha_1 = \frac{s_1}{2n}  \ ,  \quad s_1 = \sqrt{n} \ , 
		\end{align}
		we can get
		\begin{align}
			& C_{9} = \frac{4L^2}{s_{1}}C_1+\frac{4L^2}{s_{1}}C_2+ 2L^2(1-\frac{s_{1}}{n}) (1+\frac{1}{\alpha_{1}})C_3+ 2L^2(1-\frac{s_{1}}{n}) (1+\frac{1}{\alpha_{1}})C_4+  \frac{6L^2}{(1-\lambda^2)} C_7+  \frac{6L^2}{(1-\lambda^2)}C_8 \notag \\
			& \leq  \frac{4L^2}{s_{1}} \frac{3\gamma_1\eta}{\rho_{1}}+\frac{4L^2}{s_{1}} \frac{51\eta \gamma_1 L^2}{\rho_{1}\mu^2} + 2L^2\frac{2n}{s_1}\frac{ 14n\rho_{1}\gamma_1\eta}{s_{1}^2}+ 2L^2\frac{2n}{s_1}\frac{226n\rho_{1}\eta \gamma_1 L^2}{s_1^2\mu^2}  \notag \\
			& \quad +  \frac{6L^2}{(1-\lambda^2)} \frac{(1-\lambda^2)\gamma_1\eta}{6\rho_{1}}  +  \frac{6L^2}{(1-\lambda^2)} \frac{(1-\lambda^2)\eta \gamma_1 L^2}{6\rho_{1}\mu^2}  \notag \\
			& \leq  \frac{12\gamma_1\eta L^2}{s_{1}\rho_{1}}+ \frac{204\eta \gamma_1 \kappa^2L^2}{s_{1}\rho_{1}} +\frac{ 56n^2\rho_{1}\gamma_1\eta  L^2}{s_{1}^3}+ \frac{944n^2\rho_{1}\eta \gamma_1\kappa^2 L^2}{s_1^3}   +  \frac{\gamma_1\eta L^2}{\rho_{1}}  +  \frac{\eta \gamma_1 \kappa^2 L^2}{\rho_{1}}  \notag \\
			& \leq  \frac{24n\gamma_1\eta L^2}{s_{1}^2}+ \frac{408n\eta \gamma_1 \kappa^2L^2}{s_{1}^2} +\frac{ 28n\gamma_1\eta  L^2}{s_{1}^2} + \frac{472n\eta \gamma_1\kappa^2 L^2}{s_1^2}   +  \frac{2n\gamma_1\eta L^2}{s_{1}}  +  \frac{2n\eta \gamma_1 \kappa^2 L^2}{s_{1}}  \notag \\
			& \leq 936\eta \gamma_1 \kappa^2L^2  \ . 
		\end{align}
		Then, by setting
		\begin{align}
			& C_5 = \frac{22568\gamma_1 \kappa^2  L^2}{(1-\lambda^2)}  \ , \notag \\
			& C_6 = \frac{22568\gamma_1 \kappa^2L^2 }{(1-\lambda^2)}  \ ,  \notag \\
			& \eta \leq 1  \ , 
		\end{align}
		we can get
		\begin{align}
			& \quad 2\gamma_1\eta L^2+  \frac{25\eta \gamma_2 L^2}{3\mu}C_0 + 12 \eta^2 C_{9}-\frac{\eta(1-\lambda^2)}{2}C_5 \notag  \\
			& \leq 2\gamma_1\eta L^2+  \frac{25\eta \gamma_2 L^2}{3\mu} \frac{6\gamma_1  L^2}{\gamma_2\mu}  + 12 \eta^2 936\eta \gamma_1 \kappa^2L^2 -\frac{\eta(1-\lambda^2)}{2}\frac{22568\gamma_1 \kappa^2 L^2}{(1-\lambda^2)} \notag \\
			& \leq 2\eta \gamma_1 \kappa^2 L^2+  50 \eta \gamma_1  \kappa^2 L^2 + 11232\eta \gamma_1 \kappa^2L^2 - 11284\eta \gamma_1 \kappa^2 L^2 \notag \\
			& \leq  0 \ , 
		\end{align}
		and 
		\begin{align}
			& \quad 2\gamma_1\eta L^2+  \frac{25\eta \gamma_2 L^2}{3\mu} C_0+ 12 \eta^2 C_{9}-\frac{\eta(1-\lambda^2)}{2} C_6 \notag \\
			& \leq 2\gamma_1\eta L^2+  \frac{25\eta \gamma_2 L^2}{3\mu}  \frac{6\gamma_1  L^2}{\gamma_2\mu} + 12 \eta^2  936\eta \gamma_1 \kappa^2L^2 -\frac{\eta(1-\lambda^2)}{2} \frac{106\gamma_1 \kappa^2L^2 }{(1-\lambda^2)} \notag \\
			& \leq 2\gamma_1\eta L^2+ 50\gamma_1\eta\kappa^2 L^2+ 11232\eta \gamma_1 \kappa^2L^2 -11284\eta \gamma_1 \kappa^2 L^2 \notag \\
			& \leq  0 \  . 
		\end{align}
		
		In addition, we enforce
		\begin{align}
			& \quad \frac{25\eta\gamma_1^2 \kappa^2 }{6\gamma_2\mu}C_0 + 3\gamma_1^2\eta^2 C_{9} -  \frac{\gamma_1\eta}{4} \notag \\
			& \leq \frac{25\eta\gamma_1^2 \kappa^2 }{6\gamma_2\mu} \frac{6\gamma_1  L^2}{\gamma_2\mu}  + 3\gamma_1^2\eta^2 936\eta \gamma_1 \kappa^2L^2-  \frac{\gamma_1\eta}{4}  \notag \\
			& \leq 0 \ .
		\end{align}
		Due to $\eta\leq 1$,  this can be done by setting
		\begin{align}
			& \frac{25\eta\gamma_1^2 \kappa^2 }{6\gamma_2\mu} \frac{6\gamma_1  L^2}{\gamma_2\mu}  \leq  \frac{\gamma_1\eta}{8}  \ ,  \notag \\
			& 3\gamma_1^2\eta^2 936\eta \gamma_1 \kappa^2L^2\leq 3\gamma_1^2 936\eta \gamma_1 \kappa^2L^2\leq   \frac{\gamma_1\eta}{8}  \ . 
		\end{align}
		Then, we can get
		\begin{align}
			& \gamma_1 \leq  \frac{\gamma_2}{15\kappa^2 }  \ ,  \notag \\
			& \gamma_1 \leq   \frac{1}{150\kappa L}  \ . 
		\end{align}
		
		Moreover, due to $\eta\leq 1$, we enforce
		\begin{align}
			& \quad  3\gamma_2^2\eta^2C_{9}- \frac{3\gamma_2^2\eta}{4}C_0 \notag \\
			& \leq  3\gamma_2^2\eta^2936\eta \gamma_1 \kappa^2L^2 - \frac{3\gamma_2^2\eta}{4}\frac{6\gamma_1  L^2}{\gamma_2\mu}   \notag \\
			& \leq  3\gamma_2^2936\eta \gamma_1 \kappa^2L^2 - \frac{3\gamma_2^2\eta}{4}\frac{6\gamma_1  L^2}{\gamma_2\mu}   \notag \\
			& \leq 0 \ .
		\end{align}
		Then, we can get 
		\begin{align}
			& 	  \gamma_2  \leq  \frac{  1}{624  \kappa L}   \ . 
		\end{align}
		
		Furthermore, due to $\eta\leq 1$,  we enforce
		\begin{align}
			& \quad \frac{2\gamma_1^2\eta}{1-\lambda^2}C_5  + 3  \gamma_1^2\eta^2 C_{9}- \frac{1-\lambda^2}{2}C_7 \notag \\
			& \leq \frac{2\gamma_1^2\eta}{1-\lambda^2}\frac{22568\gamma_1 \kappa^2  L^2}{(1-\lambda^2)}  + 3  \gamma_1^2\eta^2 936\eta \gamma_1 \kappa^2L^2 - \frac{1-\lambda^2}{2} \frac{(1-\lambda^2)\gamma_1\eta}{6\rho_{1}} \notag \\
			& \leq  \frac{2\gamma_1^2\eta}{1-\lambda^2}\frac{22568\gamma_1 \kappa^2  L^2}{(1-\lambda^2)}  + 3  \gamma_1^2 936\eta \gamma_1 \kappa^2L^2 - \frac{1-\lambda^2}{2} \frac{(1-\lambda^2)\gamma_1\eta}{6\rho_{1}} \notag \\
			& \leq  0 \ .
		\end{align}
		Then, due to $\rho_{1} <1$, we can get
		\begin{align}
			&     \gamma_1  \leq  \frac{(1-\lambda)^2}{760\kappa L }  \ . 
		\end{align}
		
		Finally, due to $\eta\leq 1$,  we enforce
		\begin{align}
			& \quad \frac{2\eta\gamma_2^2}{1-\lambda^2}C_6+ 3  \gamma_2^2\eta^2C_{9}- \frac{1-\lambda^2}{2}C_8 \notag \\
			& \leq \frac{2\eta\gamma_2^2}{1-\lambda^2} \frac{22568\gamma_1 \kappa^2L^2 }{(1-\lambda^2)} + 3  \gamma_2^2\eta^2936\eta \gamma_1 \kappa^2L^2 - \frac{1-\lambda^2}{2}  \frac{(1-\lambda^2)\eta \gamma_1 L^2}{6\rho_{1}\mu^2} \notag \\
			& \leq \frac{2\eta\gamma_2^2}{1-\lambda^2} \frac{22568\gamma_1 \kappa^2L^2 }{(1-\lambda^2)} + 3  \gamma_2^2936\eta \gamma_1 \kappa^2L^2 - \frac{1-\lambda^2}{2}  \frac{(1-\lambda^2)\eta \gamma_1 L^2}{6\rho_{1}\mu^2} \notag \\
			& \leq  0 \ .
		\end{align}
		Then, due to $\rho_{1} <1$, we can get
		\begin{align}
			&     \gamma_2  \leq  \frac{(1-\lambda)^2}{760 L }  \ . 
		\end{align}
		
		In summary, when $t\geq1$, by setting
		\begin{align} \label{eq:hyperparameter}
			&  C_0 = \frac{6\gamma_1  L^2}{\gamma_2\mu}  \ ,   \quad  C_1 =  \frac{3\gamma_1\eta}{\rho_{1}}  \ ,    \quad C_2 = \frac{51\eta \gamma_1 L^2}{\rho_{1}\mu^2} \ ,  \quad  C_3 = \frac{ 14n\rho_{1}\gamma_1\eta}{s_{1}^2} \ , \quad  C_4 = \frac{226n\rho_{1}\eta \gamma_1 L^2}{s_1^2\mu^2}   \ ,  \notag \\
			& C_5 = \frac{22568\gamma_1 \kappa^2  L^2}{(1-\lambda^2)} \ , \quad C_6 = \frac{22568\gamma_1 \kappa^2  L^2}{(1-\lambda^2)} \ , \quad  C_7 = \frac{(1-\lambda^2)\gamma_1\eta}{6\rho_{1}}  \ , \quad   C_8 =  \frac{(1-\lambda^2)\eta \gamma_1 L^2}{6\rho_{1}\mu^2}   \ , \notag \\
			&   \rho_1=\frac{s_1}{2n}  \ , \quad  \alpha_1 = \frac{s_1}{2n}  \ ,  \quad s_1 = \sqrt{n} \ ,  \notag \\
			& \gamma_1\leq \min\left\{\frac{(1-\lambda)^2}{760\kappa L }, \frac{1}{150\kappa L} ,  \frac{\gamma_2}{15\kappa^2 } \right\} \ , \quad  \gamma_2 \leq \min\left\{ \frac{(1-\lambda)^2}{760 L },  \frac{  1}{624  \kappa L}   \right\}  \ , 
		\end{align}
		we can get
		\begin{align}
			& \quad H_{t+1} - H_{t}  \leq - \frac{\gamma_1\eta}{2} \mathbb{E}[ \|\nabla \Phi(\bar{\mathbf{x}}_{t})\|^2] - \frac{\gamma_1\eta L^2}{2}   \mathbb{E}[\|\bar{\mathbf{y}}_t -  {\mathbf{y}^*(\overline{\mathbf{x}}_t)}\|^2]\ . 
		\end{align}
		By summing $t$ from $1$ to $T-1$, we can get
		\begin{align}
			& \sum_{t=1}^{T-1}\left(\mathbb{E}[\|\nabla \Phi(\bar{\mathbf{x}}_{t})\|^2]   +L^2  \mathbb{E}[\|\bar{\mathbf{y}}_t -  {\mathbf{y}^*(\overline{\mathbf{x}}_t)}\|^2]\right) \leq \frac{2(H_{1} - \Phi(\mathbf{x}_*))}{\gamma_1\eta }  \ . 
		\end{align}

		According to the definition of the potential function, we have
		\begin{align}
			& H_{1} = \mathbb{E}[\Phi(\bar{\mathbf{x}}_{1})]+ C_0 \mathbb{E}[\|\bar{\mathbf{y}}_{1}   - \mathbf{y}^{*}(\bar{\mathbf{x}}_{1})\| ^2 ]  \notag \\
			& \quad +  \frac{C_1}{K}\sum_{k=1}^{K}\mathbb{E}[\|\nabla_{\mathbf{x}} f^{(k)}(\mathbf{x}^{(k)}_{1}, \mathbf{y}^{(k)}_{1})-\mathbf{v}_{1}^{(k)}\|^2] +  \frac{C_2}{K}\sum_{k=1}^{K}\mathbb{E}[\|\nabla_{\mathbf{y}} f^{(k)}(\mathbf{x}^{(k)}_{1}, \mathbf{y}^{(k)}_{1})-\mathbf{u}_{1}^{(k)}\|^2] \notag \\
			& \quad + \frac{C_3}{K} \sum_{k=1}^{K}\mathbb{E}[\frac{1}{n} \sum_{j=1}^{n}\|\nabla_{\mathbf{x}} f^{(k)}_{j}(\mathbf{x}_{1}^{(k)}, \mathbf{y}_{1}^{(k)})-\mathbf{g}_{j,1}^{(k)}\|^{2}]  + \frac{C_4}{K}\sum_{k=1}^{K}\mathbb{E}[\frac{1}{n} \sum_{j=1}^{n}\|\nabla_{\mathbf{y}} f^{(k)}_{j}(\mathbf{x}_{1}^{(k)}, \mathbf{y}_{1}^{(k)})-\mathbf{h}_{j,1}^{(k)}\|^{2}] \notag \\
			& \quad + \frac{C_5}{K}\sum_{k=1}^{K} \mathbb{E}[\|\bar{\mathbf{x}}_{1}-\mathbf{x}^{(k)}_{1}\|^2]+ \frac{C_6}{K}\sum_{k=1}^{K} \mathbb{E}[\|\bar{\mathbf{y}}_{1}-\mathbf{y}^{(k)}_{1}\|^2] \notag \\
			& \quad  + \frac{C_7}{K}\sum_{k=1}^{K} \mathbb{E}[\|\bar{\mathbf{a}}_{1}-\mathbf{a}^{(k)}_{1}\|^2]+ \frac{C_8}{K}\sum_{k=1}^{K}\mathbb{E}[ \|\bar{\mathbf{b}}_{1}-\mathbf{b}^{(k)}_{1}\|^2 ] \notag \\
			& \leq \mathbb{E}[\Phi(\bar{\mathbf{x}}_{0})] - \frac{\gamma_1\eta}{2} \mathbb{E}[\|\nabla \Phi(\bar{\mathbf{x}}_{0})\|^2] + \Big( \gamma_1\eta L^2 + (1-\frac{\eta\gamma_2\mu}{4})C_0 \Big)  \mathbb{E}[\|\bar{\mathbf{y}}_{0} -  \mathbf{y}^*(\bar{\mathbf{x}}_{0})\|^2]\notag \\
			& \quad + \Big( \frac{25\eta\gamma_1^2 \kappa^2 }{6\gamma_2\mu}C_0 + 3\gamma_1^2\eta^2 C_{9} -  \frac{\gamma_1\eta}{4} \Big)\mathbb{E}[\|\bar{\mathbf{v}}_{0}\|^2]    +\Big( 3\gamma_2^2\eta^2C_{9}- \frac{3\gamma_2^2\eta}{4}C_0 \Big) \mathbb{E}[\|\bar{\mathbf{u}}_{0}\|^2]    \notag \\
			& \quad  + \Big(2\gamma_1\eta L^2+  \frac{25\eta \gamma_2 L^2}{3\mu}C_0 + 12 \eta^2 C_{9}+ (1-\frac{\eta(1-\lambda^2)}{2})C_5\Big) \frac{1}{K}\sum_{k=1}^{K} \mathbb{E}[\|\bar{\mathbf{x}}_{0}-\mathbf{x}^{(k)}_{0}\|^2]  \notag \\
			& \quad + \Big(2\gamma_1\eta L^2+  \frac{25\eta \gamma_2 L^2}{3\mu} C_0+ 12 \eta^2 C_{9}+(1-\frac{\eta(1-\lambda^2)}{2} )C_6\Big) \frac{1}{K}\sum_{k=1}^{K} \mathbb{E}[\|\bar{\mathbf{y}}_{0}-\mathbf{y}^{(k)}_{0}\|^2]  \notag \\
			& \quad + \Big(2\gamma_1\eta+C_1(1-\rho_{1})^2+  \frac{6\rho_{1}^2 }{1-\lambda^2}C_7\Big)\frac{1}{K}\sum_{k=1}^{K}\mathbb{E}[\|\nabla_{\mathbf{x}} f^{(k)}(\mathbf{x}^{(k)}_{0}, \mathbf{y}^{(k)}_{0}) -\mathbf{v}^{(k)}_{0}\|^2]\notag \\
			& \quad    + \Big(\frac{25\eta \gamma_2 }{3\mu}C_0 + C_2(1-\rho_{1})^2+  \frac{6\rho_{1}^2 }{1-\lambda^2}C_8\Big) \frac{1}{K}\sum_{k=1}^{K}\mathbb{E}[\|\nabla_{\mathbf{y}} f^{(k)}({\mathbf{x}}^{(k)}_{0}, {\mathbf{y}}^{(k)}_{0})  -{\mathbf{u}}^{(k)}_{0}\|^2]   \notag \\
			& \quad +\Big(C_3(1-\frac{s_{1}}{n})(1+\alpha_{1})+\frac{ 2\rho_{1}^2}{s_{1}} C_1+ \frac{6\rho_{1}^2}{(1-\lambda^2)s_1} C_7\Big)\frac{1}{K}\sum_{k=1}^{K}\frac{1}{n}\sum_{i=1}^{n} \mathbb{E}[\|\nabla_{\mathbf{x}} f^{(k)}_{i}(\mathbf{x}_{0}^{(k)}, \mathbf{y}_{0}^{(k)})-\mathbf{g}_{i, t}^{(k)}\|^2] \notag \\
			&\quad  +\Big(C_4(1-\frac{s_{1}}{n})(1+\alpha_{1})+\frac{ 2\rho_{1}^2}{s_{1}}C_2 +\frac{6\rho_{1}^2}{(1-\lambda^2)s_1} C_8\Big) \frac{1}{K}\sum_{k=1}^{K}\frac{1}{n}\sum_{i=1}^{n} \mathbb{E}[\|\nabla_{\mathbf{y}} f^{(k)}_{i}(\mathbf{x}_{0}^{(k)}, \mathbf{y}_{0}^{(k)})-\mathbf{h}_{i, t}^{(k)}\|^2] \notag \\
			& \quad + \Big(\frac{2\gamma_1^2\eta}{1-\lambda^2}C_5  + 3  \gamma_1^2\eta^2 C_{9}+ \frac{1+\lambda^2}{2}C_7\Big) \frac{1}{K}\sum_{k=1}^{K}\mathbb{E}[\|\mathbf{a}_{0}^{(k)}-\bar{\mathbf{a}}_{0}\|^2]  \notag \\
			& \quad  + \Big(\frac{2\eta\gamma_2^2}{1-\lambda^2}C_6+ 3  \gamma_2^2\eta^2C_{9}+ \frac{1+\lambda^2}{2}C_8\Big)\frac{1}{K}\sum_{k=1}^{K} \mathbb{E}[\|\mathbf{b}_{0}^{(k)}-\bar{\mathbf{b}}_{0}\|^2]   \notag \\
			& \leq \mathbb{E}[\Phi(\bar{\mathbf{x}}_{0})] - \frac{\gamma_1\eta}{2} \mathbb{E}[\|\nabla \Phi(\bar{\mathbf{x}}_{0})\|^2] + \Big( \frac{6\gamma_1  L^2}{\gamma_2\mu}-\frac{\gamma_1  \eta L^2}{2} \Big)  \mathbb{E}[\|\bar{\mathbf{y}}_{0} -  \mathbf{y}^*(\bar{\mathbf{x}}_{0})\|^2]\notag \\
			& \quad +C_1\frac{1}{K}\sum_{k=1}^{K}\mathbb{E}[\|\nabla_{\mathbf{x}} f^{(k)}(\mathbf{x}^{(k)}_{0}, \mathbf{y}^{(k)}_{0}) -\mathbf{v}^{(k)}_{0}\|^2]  + C_2 \frac{1}{K}\sum_{k=1}^{K}\mathbb{E}[\|\nabla_{\mathbf{y}} f^{(k)}({\mathbf{x}}^{(k)}_{0}, {\mathbf{y}}^{(k)}_{0})  -{\mathbf{u}}^{(k)}_{0}\|^2]   \notag \\
			& \quad +C_3\frac{1}{K}\sum_{k=1}^{K}\frac{1}{n}\sum_{i=1}^{n} \mathbb{E}[\|\nabla_{\mathbf{x}} f^{(k)}_{i}(\mathbf{x}_{0}^{(k)}, \mathbf{y}_{0}^{(k)})-\mathbf{g}_{i, t}^{(k)}\|^2] +C_4 \frac{1}{K}\sum_{k=1}^{K}\frac{1}{n}\sum_{i=1}^{n} \mathbb{E}[\|\nabla_{\mathbf{y}} f^{(k)}_{i}(\mathbf{x}_{0}^{(k)}, \mathbf{y}_{0}^{(k)})-\mathbf{h}_{i, t}^{(k)}\|^2] \notag \\
			& \quad + C_7 \frac{1}{K}\sum_{k=1}^{K}\mathbb{E}[\|\mathbf{a}_{0}^{(k)}-\bar{\mathbf{a}}_{0}\|^2] + C_8\frac{1}{K}\sum_{k=1}^{K} \mathbb{E}[\|\mathbf{b}_{0}^{(k)}-\bar{\mathbf{b}}_{0}\|^2]    \ , 
		\end{align}
		where the second step follows from Lemma~\ref{lemma_fx},~\ref{lemma_y},~\ref{lemma_var_x},~\ref{lemma_var_y},~\ref{lemma_var_g_x},~\ref{lemma_var_g_y},~\ref{lemma_consensus_x},~\ref{lemma_consensus_y},~\ref{lemma_consensus_a},~\ref{lemma_consensus_b},~\ref{lemma_incremental}, and the last step follows from Eq.~(\ref{eq:hyperparameter}).

		On the other hand, we have
		\begin{align}\label{eq:a-consensus}
			& \quad \frac{1}{K}\sum_{k=1}^{K} \mathbb{E}[\|\mathbf{a}_0^{(k)}-\bar{\mathbf{a}}_0\|^2]  \notag \\
			& = \frac{1}{K}\sum_{k=1}^{K} \mathbb{E}[\|\mathbf{v}_0^{(k)} - \frac{1}{K}\sum_{k'=1}^{K}\mathbf{v}_0^{(k')}\|^2] \notag \\
			& = \frac{1}{K}\sum_{k=1}^{K} \mathbb{E}[\|\mathbf{v}_0^{(k)}- \nabla_{\mathbf{x}} f^{(k)}(\mathbf{x}^{(k)}_0, \mathbf{y}^{(k)}_0) + \nabla_{\mathbf{x}} f^{(k)}(\mathbf{x}^{(k)}_0, \mathbf{y}^{(k)}_0) \notag \\
			& \quad  - \frac{1}{K}\sum_{k'=1}^{K}\nabla_{\mathbf{x}} f^{(k')}(\mathbf{x}^{(k')}_0, \mathbf{y}^{(k')}_0)  + \frac{1}{K}\sum_{k'=1}^{K}\nabla_{\mathbf{x}} f^{(k')}(\mathbf{x}^{(k')}_0, \mathbf{y}^{(k')}_0)   - \frac{1}{K}\sum_{k'=1}^{K}\mathbf{v}_0^{(k')}\|^2] \notag \\
			& = \frac{1}{K}\sum_{k=1}^{K} \mathbb{E}[\|\mathbf{v}_0^{(k)}- \nabla_{\mathbf{x}} f^{(k)}(\mathbf{x}^{(k)}_0, \mathbf{y}^{(k)}_0) \|^2]+ \frac{1}{K}\sum_{k=1}^{K} \mathbb{E}[\|\nabla_{\mathbf{x}} f^{(k)}(\mathbf{x}^{(k)}_0, \mathbf{y}^{(k)}_0)  - \frac{1}{K}\sum_{k'=1}^{K}\nabla_{\mathbf{x}} f^{(k')}(\mathbf{x}^{(k')}_0, \mathbf{y}^{(k')}_0)\|^2]  \notag \\
			& \quad +\frac{1}{K}\sum_{k=1}^{K} \mathbb{E}[\| \frac{1}{K}\sum_{k'=1}^{K}\nabla_{\mathbf{x}} f^{(k')}(\mathbf{x}^{(k')}_0, \mathbf{y}^{(k')}_0)   - \frac{1}{K}\sum_{k'=1}^{K}\mathbf{v}_0^{(k')}\|^2] \notag \\
			& \leq  \frac{2}{K}\sum_{k=1}^{K} \mathbb{E}[\|\mathbf{v}_0^{(k)}- \nabla_{\mathbf{x}} f^{(k)}(\mathbf{x}^{(k)}_0, \mathbf{y}^{(k)}_0) \|^2  \  , \notag \\
			& \quad \frac{1}{K}\sum_{k=1}^{K} \mathbb{E}[\|\mathbf{b}_0^{(k)}-\bar{\mathbf{b}}_0\|^2]  \leq  \frac{2}{K}\sum_{k=1}^{K} \mathbb{E}[\|\mathbf{u}_0^{(k)}- \nabla_{\mathbf{y}} f^{(k)}(\mathbf{x}^{(k)}_0, \mathbf{y}^{(k)}_0) \|^2  \  ,
		\end{align}
		where we used $\nabla_{\mathbf{x}} f^{(k)}(\mathbf{x}^{(k)}_0, \mathbf{y}^{(k)}_0) = \frac{1}{K}\sum_{k'=1}^{K}\nabla_{\mathbf{x}} f^{(k')}(\mathbf{x}^{(k')}_0, \mathbf{y}^{(k')}_0)$ and $\nabla_{\mathbf{y}} f^{(k)}(\mathbf{x}^{(k)}_0, \mathbf{y}^{(k)}_0)  = \frac{1}{K}\sum_{k'=1}^{K}\nabla_{\mathbf{y}} f^{(k')}(\mathbf{x}^{(k')}_0, \mathbf{y}^{(k')}_0)$ in the proof. 
		
		Then, based on Lemma~\ref{lemma_var_x},~\ref{lemma_var_y},~\ref{lemma_var_g_x},~\ref{lemma_var_g_y} when $t=0$, we can obtain
		\begin{align}
			& H_{1} \leq \mathbb{E}[\Phi(\bar{\mathbf{x}}_{0})]+  \frac{6\gamma_1  L^2}{\gamma_2\mu} \mathbb{E}[\|\bar{\mathbf{y}}_{0}   - \mathbf{y}^{*}(\bar{\mathbf{x}}_0)\| ^2 ] - \frac{\gamma_1\eta}{2} \mathbb{E}[\|\nabla \Phi(\bar{\mathbf{x}}_{0})\|^2]  - \frac{\gamma_1  \eta L^2}{2}  \mathbb{E}[\|\bar{\mathbf{y}}_{0} -  \mathbf{y}^*(\bar{\mathbf{x}}_{0})\|^2]\notag \\
			& \quad +  \frac{4\gamma_1\eta}{\rho_{1}} \frac{n-s_0}{(n-1)s_0}\frac{1}{K}\sum_{k=1}^{K}\frac{1}{n}\sum_{i=1}^{n} \mathbb{E}[ \|\nabla_{\mathbf{x}} f^{(k)}_i(\mathbf{x}_{0}^{(k)}, \mathbf{y}_{0}^{(k)})\|^2] \notag \\
			& \quad +  \frac{52\eta \gamma_1 L^2}{\rho_{1}\mu^2}\frac{n-s_0}{(n-1)s_0}\frac{1}{K}\sum_{k=1}^{K}\frac{1}{n}\sum_{i=1}^{n}  \mathbb{E}[\|\nabla_{\mathbf{y}} f^{(k)}_i(\mathbf{x}_{0}^{(k)}, \mathbf{y}_{0}^{(k)})\|^2 ] \notag \\
			& \quad +\frac{ 14n\rho_{1}\gamma_1\eta}{s_{1}^2}(1-\frac{s_{0}}{n}) \frac{1}{K} \sum_{k=1}^{K} \frac{1}{n}\sum_{i=1}^{n} \mathbb{E}[\|\nabla_{\mathbf{x}} f^{(k)}_{i}(\mathbf{x}_{0}^{(k)}, \mathbf{y}_{0}^{(k)})\|^{2}] \notag \\
			& \quad  + \frac{226n\rho_{1}\eta \gamma_1 L^2}{s_1^2\mu^2}  (1-\frac{s_{0}}{n}) \frac{1}{K}\sum_{k=1}^{K}\frac{1}{n}\sum_{i=1}^{n} \mathbb{E}[\|\nabla_{\mathbf{y}} f^{(k)}_{i}(\mathbf{x}_{0}^{(k)}, \mathbf{y}_{0}^{(k)})\|^{2}]\notag \\
			& \leq  \mathbb{E}[\Phi(\bar{\mathbf{x}}_{0})]+  \frac{6\gamma_1  L^2}{\gamma_2\mu} \mathbb{E}[\|\bar{\mathbf{y}}_{0}   - \mathbf{y}^{*}(\bar{\mathbf{x}}_0)\| ^2 ] - \frac{\gamma_1\eta}{2} \mathbb{E}[\|\nabla \Phi(\bar{\mathbf{x}}_{0})\|^2]   - \frac{\gamma_1  \eta L^2}{2}  \mathbb{E}[\|\bar{\mathbf{y}}_{0} -  \mathbf{y}^*(\bar{\mathbf{x}}_{0})\|^2]\notag \\
			& \quad + 15\gamma_1\eta\frac{n-s_0}{s_0s_1} \frac{1}{K}\sum_{k=1}^{K}\frac{1}{n}\sum_{i=1}^{n}  \mathbb{E}[\|\nabla_{\mathbf{x}} f^{(k)}_i(\mathbf{x}_{0}^{(k)}, \mathbf{y}_{0}^{(k)})\|^2]   \notag \\
			& \quad +278\kappa^2\gamma_1\eta \frac{n-s_0}{s_0s_1}  \frac{1}{K}\sum_{k=1}^{K}\frac{1}{n}\sum_{i=1}^{n} \mathbb{E}[\|\nabla_{\mathbf{y}} f^{(k)}_{i}(\mathbf{x}_{0}^{(k)}, \mathbf{y}_{0}^{(k)})\|^{2}]  \ , 
		\end{align}
		where  the last step follows from $s_0\leq n$ and $\rho_1=\frac{s_1}{2n}$. 
		
		Finally, we can get
		\begin{align}
			&  \frac{1}{T}\sum_{t=0}^{T-1}\left(\mathbb{E}[\|\nabla \Phi(\bar{\mathbf{x}}_{t})\|^2]   +L^2  \mathbb{E}[\|\bar{\mathbf{y}}_t -  {\mathbf{y}^*(\overline{\mathbf{x}}_t)}\|^2]\right)\leq \frac{2(\Phi({\mathbf{x}}_{0})- \Phi(\mathbf{x}_*))}{\gamma_1\eta T}  +  \frac{12 L^2}{\gamma_2\eta T \mu} \mathbb{E}[\|\bar{\mathbf{y}}_{0}   - \mathbf{y}^{*}(\bar{\mathbf{x}}_0)\| ^2 ]  \notag \\
			& \quad + \frac{30}{ T} \frac{n-s_0}{s_0s_1} \frac{1}{K}\sum_{k=1}^{K}\frac{1}{n}\sum_{i=1}^{n}  \mathbb{E}[\|\nabla_{\mathbf{x}} f^{(k)}_i(\mathbf{x}_{0}^{(k)}, \mathbf{y}_{0}^{(k)})\|^2]  +\frac{556\kappa^2}{ T} \frac{n-s_0}{s_0s_1} \frac{1}{K}\sum_{k=1}^{K}\frac{1}{n}\sum_{i=1}^{n} \mathbb{E}[\|\nabla_{\mathbf{y}} f^{(k)}_{i}(\mathbf{x}_{0}^{(k)}, \mathbf{y}_{0}^{(k)})\|^{2}]  \  . 
		\end{align}

	\end{proof}

\end{document}